\documentclass[journal,10pt,twocolumn]{IEEEtran}

\usepackage{graphicx}
\usepackage{amsmath}
\usepackage{amssymb}
\usepackage{algorithmic}
\usepackage[ruled]{algorithm2e}
\usepackage{epsfig}
\usepackage{mathrsfs}
\usepackage{multirow}
\usepackage{booktabs}
\usepackage{mathtools}
\usepackage{bm}

\usepackage{subfigure}

\newtheorem{theorem}{Theorem}

\ifCLASSINFOpdf
  \DeclareGraphicsExtensions{.pdf,.jpeg,.png}
\else
  \DeclareGraphicsExtensions{.eps}
\fi

\begin{document}
%
\title{Active Self-Paced Learning for Cost-Effective and Progressive Face Identification}
%
%
%

\author{Liang Lin, Keze Wang, Deyu Meng, Wangmeng Zuo, and Lei Zhang
\IEEEcompsocitemizethanks
 {

This work was supported in part by the
State Key Development Program under Grant 2016YFB1001004, in part by
the National Natural Science Foundation of China under Grant 61671182, 61661166011 and 61373114, in part by the National Grand Fundamental Research 973 Program of China under Grant No. 2013CB329404, in part by Hong Kong Scholars Program and Hong Kong Polytechnic University Mainland University Joint Supervision Scheme, and sponsored
by CCF-Tencent Open Research Fund (NO. AGR20160115).

 \IEEEcompsocthanksitem
 L. Lin and K. Wang are with School of Data and Computer Science, Sun Yat-sen University, Guangzhou, China and also with Engineering Research Center for Advanced Computing Engineering Software of Ministry of Education, China. Email: linliang@ieee.org; kezewang@gmail.com.

\IEEEcompsocthanksitem D. Meng is with School of Mathematics and Statistics and Ministry of  Education Key Lab of  Intelligent Networks and Network Security, Xi'an Jiaotong University, P. R. China. Email: dymeng@mail.xjtu.edu.cn.

\IEEEcompsocthanksitem W. Zuo is with School of Computer Science and Technology, Harbin Institute of Technology, Harbin, P. R. China. Email: cswmzuo@gmail.com.

\IEEEcompsocthanksitem L. Zhang is with Dept. of Computing, The Hong Kong Polytechnic University, Hong Kong. Email: cslzhang@comp.polyu.edu.hk.}}

\markboth{}
{}


\IEEEcompsoctitleabstractindextext{
\begin{abstract}

This paper aims to develop a novel cost-effective framework for face identification, which progressively maintains a batch of classifiers with the increasing face images of different individuals. By naturally combining two recently rising techniques: active learning (AL) and self-paced learning (SPL), our framework is capable of automatically annotating new instances and incorporating them into training under weak expert recertification. We first initialize the classifier using a few annotated samples for each individual, and extract image features using the convolutional neural nets. Then, a number of candidates are selected from the unannotated samples for classifier updating, in which we apply the current classifiers ranking the samples by the prediction confidence. In particular, our approach utilizes the high-confidence and low-confidence samples in the self-paced and the active user-query way, respectively. The neural nets are later fine-tuned based on the updated classifiers. Such heuristic implementation is formulated as solving a concise active SPL optimization problem, which also advances the SPL development by supplementing a rational dynamic curriculum constraint. The new model finely accords with the ``instructor-student-collaborative'' learning mode in human education. The advantages of this proposed framework are two-folds: i) The required number of annotated samples is significantly decreased while the comparable performance is guaranteed. A dramatic reduction of user effort is also achieved over other state-of-the-art active learning techniques. ii) The mixture of SPL and AL effectively improves not only the classifier accuracy compared to existing AL/SPL methods but also the robustness against noisy data. We evaluate our framework on two challenging datasets, which include hundreds of persons under diverse conditions, and demonstrate very promising results. Please find the code of this project at: http://hcp.sysu.edu.cn/projects/aspl/

\end{abstract}



}

\maketitle

\IEEEdisplaynotcompsoctitleabstractindextext

\IEEEpeerreviewmaketitle

\section{Introduction}\label{sec:introduction}

With the growth of mobile phones, cameras and social networks, a large amount of photographs is rapidly created, especially those containing person faces. { To interact with these photos, there have been increasing demands of developing intelligent systems (e.g., content-based personal photo search and sharing from either his/her mobile albums or social network) with face recognition techniques~\cite{celli2014automatic,
stone2010toward, sid15tcsvt}. Thanks to several recently proposed pose/expression normalization and alignment-free approaches~\cite{pfr13pami, tpr13cvpr, hpen15cvpr}, identifying face in the wild has achieved remarkable progress. As for the commercial product,} the website ``Face.com'' { once provided} an API (application interface) to automatically detect and recognize faces in photos. The main problem in such scenarios is to identify individuals from images under a relatively unconstrained environment. Traditional methods usually handle this problem by supervised learning \cite{face_hybrid}, while it is typically expensive and time-consuming to prepare a good set of labeled samples. Since only a few data are labeled, Semi-supervised learning~\cite{AUSDL15ICCV} may be a good candidate to solve this problem. But it has been pointed out by~\cite{hurtdata_pami15}: Due to large amounts of noisy samples and outliers, directly using the unlabeled data may significantly reduce learning performance.

This paper targets on the challenge of incrementally learning a batch of face recognizers with the increasing face images of different individuals\footnote{http://hcp.sysu.edu.cn/projects/aspl/}. Here we assume that the person faces can be basically detected and localized by existing face detectors. However, to build such a system is quite challenging in the following aspects.

\begin{itemize}
\item Person faces have large appearance variations (see examples in Fig. \ref{fig:instances} (a)) caused by diverse views and expressions as well as facial accessories (e.g., glasses and hats) and aging. The different lighting condition is also required to be considered in practice.
\item It is possible that only a few labeled samples are accessible at first, and the changes of personal faces are rather unpredictable over time, especially under the current scenarios that there are large amount of images swarmed into Internet every day.
\item Even though a few user interventions (e.g., labeling new samples) could be allowed, the user effort is desired to be kept minimizing over time.
\end{itemize}

Conventional incremental face recognition methods such as incremental subspace approaches \cite{iPCA,kim2007incremental} often fail on complex and large-scale environments. Their performances could be dropped drastically when the initial training set of face images is either insufficient or inappropriate. In addition, most of existing incremental approaches suffer from noisy samples or outliers in the model updating. In this work, we propose a novel active self-paced learning framework (ASPL) to handle the above difficulties, which absorbs powers of two recently rising techniques: active learning (AL) \cite{con_AL, keze_CEAL} and self-paced learning (SPL)~\cite{spcl,spld,spl_reranking}. In particular, our framework tends to conduct a ``Cost-less-Earn-more'' working manner: as much as possible pursuing a high performance while reducing costs.

The basic approach of the AL methods is to progressively select and annotate most informative unlabeled samples to boost the model, in which user interaction is allowed. The sample selection criteria is the key in AL, and it is typically defined according to the classification uncertainty of samples. Specifically, the samples of low classification confidence, together with other informative criteria like diversity, are generally treated as good candidates for model retraining. On the other hand, SPL is a recently proposed learning regime to mimic the learning process of humans/animals that gradually incorporates easy to more complex samples into training \cite{curriculun_learning,spl_kumar}, where an easy sample is actual the one of high classification confidence by the currently trained model. Interestingly, the two categories of learning methods select samples with the opposite criteria. This finding inspires us to investigate the connection between the two learning regimes and the possibility of making them complementary to each other. {Moreover, as pointed out in \cite{sid15tcsvt, Hu_2015_ICCV_Workshops}, learning based features are considered to be able to exploit information with better discriminative ability for face recognition, compared to the hand-crafted features. We thus utilize the deep convolutional neural network (CNN) \cite{lecun2010convolutional, keze_dpl} for feature extraction instead of using handcraft image features.}. In sum, we aim at designing a cost-effective and progressive learning framework, which is capable of automatically annotating new instances and incorporating them into training under weak expert recertification. In the following, we discuss the advantage of our ASPL framework in two aspects:  ``Cost-less'' and ``Earn-more''.

\begin{figure}
\center
\label{fig:instances}
\subfigure[]{
\begin{minipage}[b]{0.17\textwidth}
\raggedleft\includegraphics[width=1\textwidth]{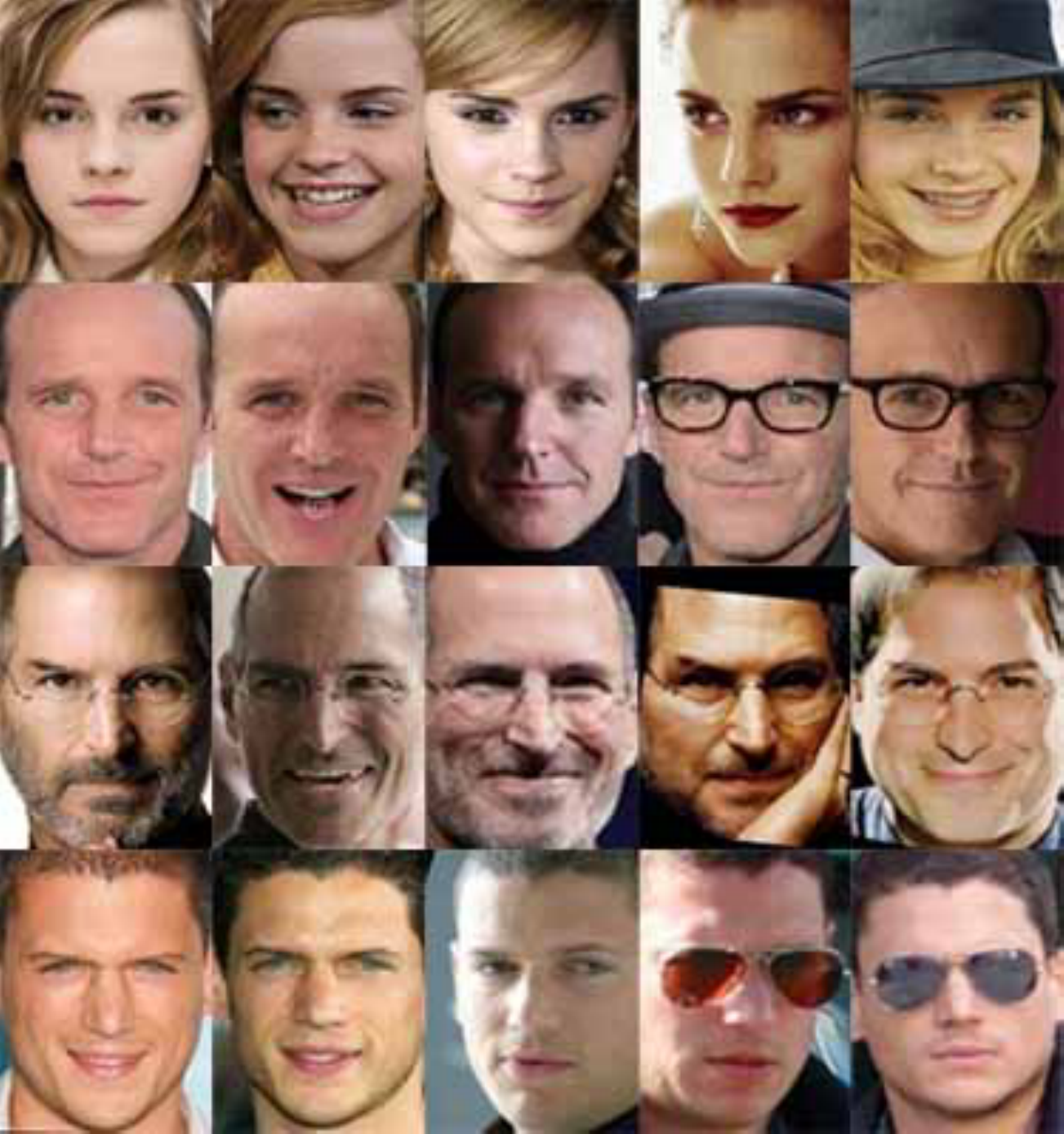}\\
\end{minipage}
}
\subfigure[]{
\begin{minipage}[b]{0.20\textwidth}
\raggedright\includegraphics[width=1\textwidth]{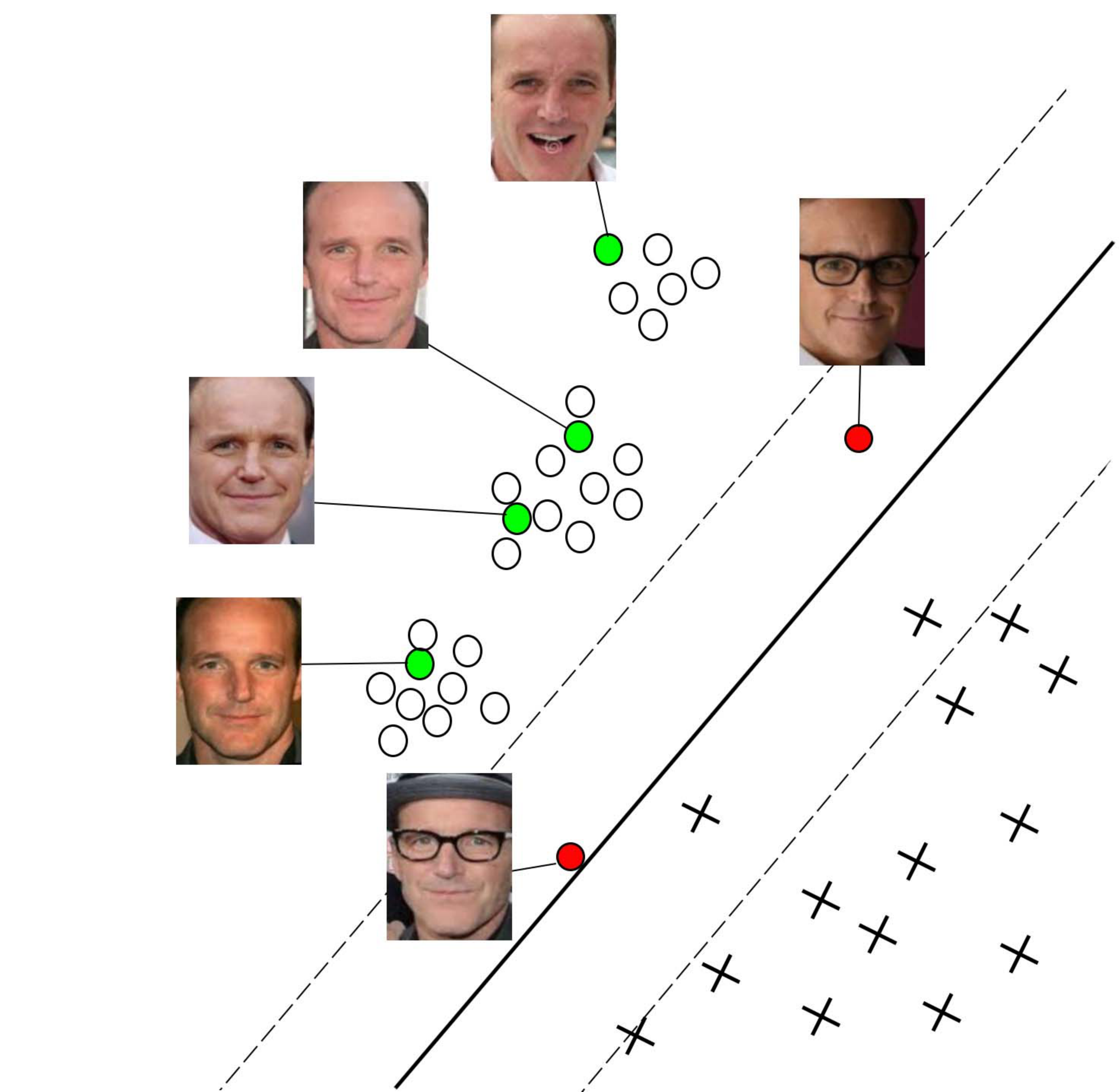}
\end{minipage}
}
\caption{Illustration of high- and low-confidence samples in the feature space. (a) shows a few face instances of different individuals, and these instances have large appearance variations. (b) illustrates how the samples distribute in the feature space, where samples of high classification confidence distribute compactly to form several clusters and low confidence samples are scattered and close to the classifier decision boundary.}
\end{figure}

(I) \textbf{Cost less}: Our framework is capable of building effective classifiers with less labeled training instances and less user efforts, compared with other state-of-the-art algorithms. This property is achieved by combining the active learning and self-paced learning in the incremental learning process. In certain feature space of model training as Fig. 1 (b) illustrates, samples of low classification confidence are scattered and close to the classifier decision boundary while high confidence samples distribute compactly in the intra-class regions. Our approach takes both categories of samples into consideration for classifier updating. The benefit of this strategy includes: i) High-confidence samples can be automatically labeled and consistently added into model training throughout the learning process in a self-paced fashion, particularly when the classifier becomes more and more reliable at later learning iterations. This significantly reduce the burden of user annotations and make the method scalable in large-scale scenarios. ii) The low-confidence samples are selected by allowing active user annotations, making our approach more efficiently pick up informative samples, more adapt to practical variations and converge faster, especially in the early learning stage of training.

(II) \textbf{Earn more}: The mixture of self-paced learning and active learning effectively improves not only the classifier accuracy but also the classifier robustness against noisy samples. From the perspective of AL, extra high-confidence samples are automatically incorporated into the retraining without cost of human labor in each iteration, and faster convergence can be thus gained. These introduced high-confidence samples also contribute to suppress noisy samples in learning, due to their compactness and consistency in the feature space. From the SPL perspective, allowing active user intervention generates the  reliable and diverse samples that can avoid the learning been misled by outliers. In addition, utilizing the CNN facilitates to pursue a higher classification performance by learning the convolutional filters instead of hand-craft feature engineering.

In brief, our ASPL framework includes two main phases. At the initial stage, we first learn a general face representation using an architecture of convolutional neural nets, and train a batch of classifiers with a very small set of annotated samples of different individuals. In the iteration learning stage, we rank the unlabeled samples according to how they relate to the current classifiers, and retrain the classifiers by selecting and annotating samples in either active user-query or self-paced manners. We can also make the CNN fine-tuned based on the updated classifiers.

The key point in designing such an effective interactive learning system is to make an efficient labor division between computers and human participants, i.e., we should possibly feed computable and faithful tasks into computers, and to possibly arrange labor-saving and intelligent tasks to humans \cite{ISed}. The proposed ASPL framework provides a rational realization to this task by automatically distinguishing high-confidence samples, which can be easily and faithfully recognized by computers in a self-paced way, and low-confidence ones, which can be discovered by requesting user annotation.

The main {\bf contributions} of this work are several folds. i) To the best of our knowledge, our work is the first one to make a face recognition framework capable of automatically annotating high-confidence samples and involve them into training without need of extra human labor in a purely self-paced manner under weak recertification of active learning. Especially in that along the learning process, we can achieve more and more pseudo-labeled samples to facilitate learning totally for free. Our framework is thus suitable in practical large-scale scenarios. The proposed framework can be easily extended to other similar visual recognition tasks. ii) We provide a concise optimization problem and theoretically interpret that the proposed ASPL is an rational implementation for solving this problem. iii) This work also advances the SPL development, by setting a dynamic curriculum variation. The new SPL setting better complies with the ``instructor-student-collaborative'' learning mode in human education than previous models. iv) Extensive experiments on challenging CACD and CASIA-WebFace datasets show that our approach is capable of achieving competitive or even better performance under only small fraction of sample annotations than that under overall labeled data. A dramatic reduction ($> 30 \%$) of user interaction is achieved over other state-of-the-art active learning methods.

The rest of the paper is organized as follows. Section II presents a brief review of related work. Section III overview the pipeline of our framework, followed by a discussion of model formulation and optimization in Section IV. The experimental results, comparisons and component analysis are presented in Section V. Section VI concludes the paper.

\section{Related Work}
\label{sec:related_work}

In this section, we first present a review for the incremental face recognition, and then briefly introduce related developments on active learning and self-paced learning.

{\bf Incremental Face Recognition. } There are two categories of methods addressing the problem of identifying faces with incremental data, namely incremental subspace and incremental classifier methods. The first category mainly includes the incremental versions of traditional subspace learning approaches such as principal component analysis (PCA) \cite{smith2002tutorial} and linear discriminant analysis (LDA) \cite{kim2007incremental}. These approaches map facial features into a subspace, and keep the eigen representations (i.e., eigen-faces) updated by incrementally incorporating new samples. And face recognition is commonly accomplished by the nearest neighbor-based feature matching, which is computational expensive when a large number of samples are accumulated over time. On the other hand, the incremental classifier methods target on updating the prediction boundary with the learned model parameters and new samples. Exemplars include the incremental support vector machines (ISVM) \cite{ISVM} and the online sequential forward neural network \cite{TNN-2006}. In addition, several attempts have been made to absorb advantages from both of the two categories of methods. For example, Ozawa et al., \cite{ozawa2005incremental} proposed to integrate the Incremental PCA with the resource allocation network in an iterative way. Although these mentioned approaches make remarkable progresses, they suffer from low accuracy compared with those of batch-based state-of-the-art face recognizers, and none of these approaches have been successfully validated on large-scale datasets (e.g., more than 500 individuals). And these approaches are basically studied in the context of fully supervised learning, i.e., both initial and incremental data are required to be labeled.

{\bf Active Learning.} This branch of works mainly focus on actively selecting and annotating the most informative unlabeled samples, in order to avoid unnecessary and redundant annotation. The key part of active learning is thus the selection strategy, i.e., which samples should be presented to the user for annotation.  One of the most common strategies is the certainty-based selection \cite{lewis1994sequential,tong2002support}, in which the certainties are measured according to the predictions on new unlabeled samples obtained from the initial classifiers. For example, Lewis et al., \cite{lewis1994sequential} proposed to take the most uncertain instance as the one that has the largest entropy on the conditional distribution over its predicted labels. Several SVM-based methods \cite{tong2002support} determine the uncertain samples as they are relatively close to the decision boundary. The sample certainty was also measured by applying a committee of classifiers in \cite{mccallumzy1998employing}.  These certainty-based approaches usually ignore the large set of unlabeled instances, and are thus sensitive to outliers. A number of later methods present the information density measure by exploiting the information of unlabeled data when selecting samples. For example, the informative samples are sequentially selected to minimize the generalization error of the trained classifier on the unlabeled data, based on a statistical approach \cite{joshi2009multi} or prior information \cite{kapoor2009faces}. In \cite{kapoor2007active,li2013adaptive}, instances are taken to maximize the increase of mutual information between the candidate instances and the remaining ones based on Gaussian Process models. The diversity of the selected instance over the unlabeled data has been also taken into consideration \cite{brinker2003incorporating}. Recently, Elhamifar et al., \cite{con_AL} presented a general framework via convex programming, which considered both the uncertainty and diversity measure for sample selection. However, these mentioned active learning approaches usually emphasize those low-confidence samples (e.g., uncertain or diverse samples) while ignoring the other majority of high-confidence samples. {To enhance the discriminative capability, wang~\cite{AUSDL15ICCV} et al. proposed a unified semi-supervised learning framework, which incorporates the high confidence coding vectors of unlabeled data into training under the proposed effective iterative algorithm, and demonstrate its effectiveness in dictionary-based classification. Our work inspires by this work, and also employs the high-confidence samples to improve both accuracy and robustness of classifiers.}

\begin{figure*}[!ht]
\begin{center}
\includegraphics[width=\textwidth]{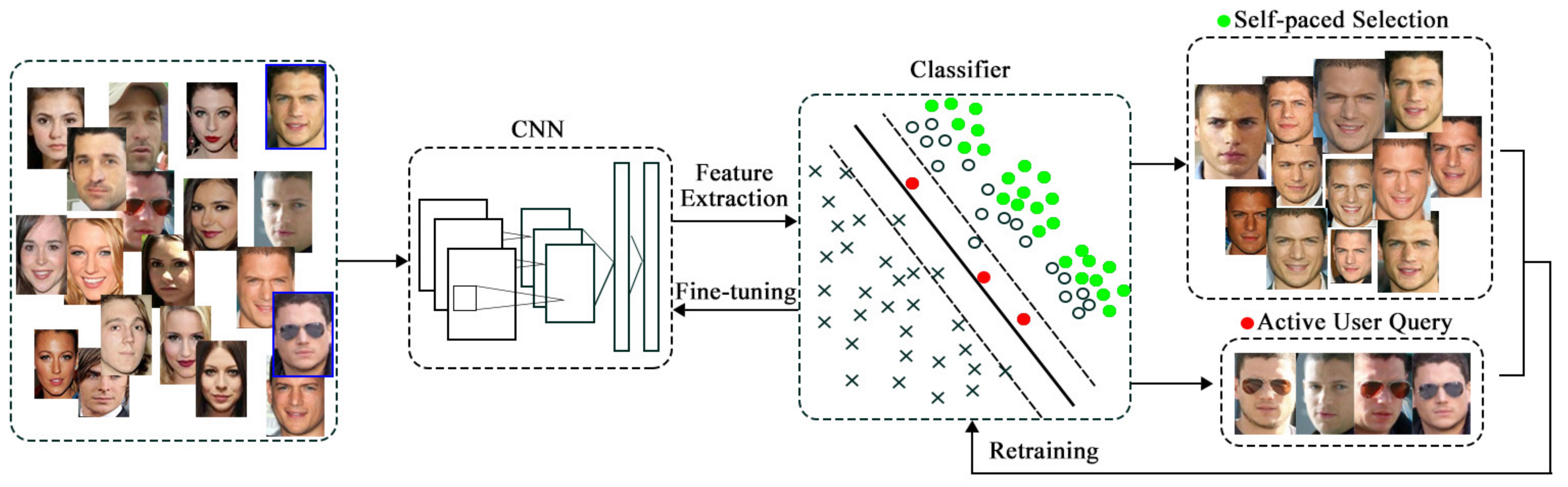}
\vspace{-20pt}
\caption{Illustration of our proposed cost-effective framework. The pipeline includes stages of CNN and model initialization; classifier updating; high-confidence sample labeling by the SPL, low-confidence sample annotating by AL and CNN fine-tuning, where the arrows represent the workflow. The images highlighted by blue in the left panel represent the initially selected samples.}
\label{fig:overview}
\vspace{-10pt}
\end{center}
\end{figure*}

{\bf Self-paced Learning.} Inspired by the cognitive principle of humans/animals, Bengio et al. \cite{curriculun_learning} initialized the concept of curriculum learning (CL), in which a model is learned by gradually including samples into training from easy to complex. To make it more implementable,  Kumar et al. \cite{spl_kumar} substantially prompted this learning philosophy by formulating the CL principle as a concise optimization model named self-paced learning (SPL). The SPL model includes a weighted loss term on all samples and a general SPL regularizer imposed on sample weights. By sequentially optimizing the model with gradually increasing pace parameter on the SPL regularizer, more samples can be automatically discovered in a pure self-paced way. Jiang et al. \cite{spcl,spmf,spl_reranking} provided more comprehensive understanding for the learning insight underlying SPL/CL, and formulated the learning model as a general optimization problem as:
\begin{equation}
\label{eq:spcl_obj}
\begin{split}
\!\min_{\mathbf{w},\mathbf{v}\in \lbrack 0,\!1]^{n}}
\sum_{i=1}^n v_i L(\mathbf{w}; \mathbf{x}_i, y_i) + f(\textbf{v};\lambda) \\ \text{  s.t. } \mathbf{v} \in {\bm \Psi}
\end{split}
\end{equation}
where $\mathcal{D} = \{(x_i,y_i)\}_{i=1}^n$ corresponds to the training dataset,
$L(\mathbf{w}; \mathbf{x}_i, y_i)$ denotes the loss function which calculates the cost between the objective label $y_i$ and the estimated one, $\mathbf{w}$ represents the model parameter inside the decision function, $\mathbf{v=[}v_{1},v_{2},\cdots ,v_{n}\mathbf{]}^{T}$ denote the weight variables reflecting the samples' importance. $\lambda$ is a parameter for controlling the learning pace, which is also referred as ``pace age''.

In the model, $f(\textbf{v};\lambda)$ corresponds to a self-paced regularizer. Jiang et al. abstracted three necessary conditions it should be satisfy \cite{spl_reranking,spcl}:
(1) $f(v;\lambda)$ is convex with respect to $v \in [0,1]$;
(2) The optimal weight of each sample should be monotonically decreasing with respect to its corresponding loss; and (3) The optimal weight of each sample should be monotonically decreasing with respect to the pace parameter $\lambda$.

In this axiomic definition, Condition 2 indicates that the model inclines to select easy samples (with smaller errors) in favor of complex samples (with larger errors). Condition 3 states that when the model ``age'' $\lambda$ gets larger, it embarks on incorporating more, probably complex, samples to train a ``mature'' model. The convexity in Condition 1 further ensures that the model can find good solutions.

$\Psi$ is the so called curriculum region that encodes the information of predetermined curriculums. Its axiomic definition contains two conditions \cite{spcl}: (1) It should be nonempty and convex; and (2) If $x_i$ is ranking before $x_j$ in curriculum (more important for the problem), the expectation $\int_{\Psi} v_i\,d\mathbf{v}$ should be larger than $\int_{\Psi} v_j\,d\mathbf{v}$. Condition 1 ensures the soundness for the calculation of this specific constraint, and Condition 2 indicates that samples to be learned earlier is supposed to have larger expected values.
This constraint weakly implies a prior learning sequence of samples, where the expected value for the favored samples should be larger.

The SPL model (\ref{eq:spcl_obj}) finely simulates the learning process of human education. Specifically, it builds an ``instructor-student collaborative'' paradigm, which on one hand utilizes prior knowledge provided by instructors as a guidance for curriculum designing (encoded by the curriculum constraint), and on the other hand leaves certain freedom to students to ameliorate the actual curriculum according to their learning pace (encoded by the self-paced regularizer). Such a model not only includes all previous SPL/CL methods as its special cases, but also
provides a general guild line to extend a rational SPL implementation scheme against certain learning task. Based on this framework, multiple SPL variations have been recently proposed, like SPaR \cite{spl_reranking}, SPLD \cite{spld}, SPMF \cite{spmf} and SPCL \cite{spcl}.

The SPL related strategies have also been recently attempted in a series of applications, such as specific-class segmentation learning~\cite{spl_kumar_segment}, visual category discovery~\cite{spl_Kristen}, long-term tracking~\cite{spl-tracking}, action recognition~\cite{spld} and background subtraction~\cite{spmf}. Especially, the SPaR method, constructed based on the general formulation (\ref{eq:spcl_obj}), was applied to the challenging SQ/000Ex task of the TRECVID MED/MER competition, and
achieved the leading performance among all competing teams~\cite{MED14}.

{\bf Complementarity between AL and SPL:}
It is interesting that the function of SPL is very complementary to that of AL. The SPL methods emphasize easy samples in learning, which correspond to the high-confidence intra-class samples, while AL inclines to pick up the most uncertain and informative samples for the learning task, which are always located in low-confidence area near classification boundaries. SPL is capable of easily attaining large amount of faithful pseudo-labeled samples with less requirement of human labors (by reranking technique~\cite{spl_reranking}. We will introduce details in Section 4), while tends to underestimate the roles of those most informative ones intrinsically configuring the classification boundaries; on the contrary, AL inclines to get informative samples, while need more human labors to manually annotate these samples with more carefully annotation.
We thus expect to effectively mix these two learning schemes to help incremental learning both improve the efficiency with less human labors (i.e., Cost Less) and achieve better accuracy and robustness of the learned classifier against noisy samples (i.e., Earn More). This constructs the basic motivation of our ASPL framework for face identification under large-scale scenarios.

\section{Framework Overview}
\label{sec:dlem_framework}

In this section, we illustrate how our ASPL model works. As illustrated in Fig. \ref{fig:overview}, the main stages in our framework pipeline include: CNN pretraining for face representation, classifier updating, high-confidence sample pseudo-labeling in a self-paced fashion, low-confidence sample annotating by active users, and CNN fine-tuning.

\emph{CNN pretraining}: Before running the ASPL framework, we need to pretrain a CNN for feature extraction based on a pre-given face dataset. These images are extra selected without overlapping to all our experimental data. {
Since several public available CNN architectures~\cite{alexnet, vgg19} have achieved remarkable success on visual recognition,  our framework supports to directly employ these architectures and their pretrained model as initialized parameters. In our all experiments, AlexNet~\cite{alexnet} is utilized. Given the extra selected of annotated samples, we further fine-tune the CNN for learning discriminative feature representation.
}

\emph{Initialization}: At the beginning, we randomly select few images for each individual, extract feature representation for them by pretrained CNN, and manually annotate labels to them as the starting point.

\emph{Classifier updating}: In our ASPL framework, we use one-vs-all linear SVM  as our classifier updating strategies. In the beginning, only a small part of samples are labeled, and we train an initial a classifier for every individual using these samples. As the framework gets mature, samples manually annotated by the AL and pseudo-labeled by the SPL are growing, we adopt them to retrain the classifiers.

\emph{High-confidence sample pseudo-labeling}: We rank the unlabeled samples by their important weights via the current classifiers, e.g., using the classification prediction hinge loss, and then assign pseudo-labels to the top-ranked samples of high confidences. This step can be automatically implemented by our system.

\emph{Low-confidence sample annotating}: Based on certain AL criterion obtained under the current classifiers, rank all unlabeled samples, select those top-ranked ones (most informative and generally with low-confidence) from the unlabeled samples, and then manually annotate these samples by active users.

\emph{CNN fine-tuning}: After several steps of the interaction, we make the neural nets fine-tuned by the backward propagation algorithm. All self-labeled samples by the SPL and manually annotated ones by the AL are added into the network, we utilize the softmax loss to optimize the CNN parameters via stochastic gradient decent approach.



\section{Formulation and Optimization}

In this section we will discuss the formulation of our proposed framework, and also provide a theoretical interpretation of its entire pipeline from the perspective of optimization. In specific, we can theoretically justify that the entire pipeline of this framework finely accords with a solving process for an active self-paced learning (ASPL) optimization model. Such a theoretical understanding will help deliver more insightful understanding on the intrinsic mechanism underlying the ASPL system.

\subsection{Active Self-paced Learning}
\label{sec:alg}
In the context of face identification, suppose that we have $n$ facial photos which are taken from $m$ subjects. Denote the training samples as $\mathcal{D}=\{
\mathbf{x}_{i}\}_{i=1}^{n} \subset R^{d}$, where $\mathbf{x}_{i}$ is the $d$-dimensional feature representation for the $i$th sample. We have $m$ classifiers for recognizing each sample by the one-vs-all strategy.

Learned knowledge from data will be utilized to ameliorate our model after a period of pace increasing. Correspondingly, we denote the label set of $\mathbf{x}_i$ as $\mathbf{y}_i = \{y_{i}^{(j)} \in \{-1,1\}\}_{j=1}^m$, where $y_{i}^{(j)}$ corresponds to the label of $\mathbf{x}_{i}$ for the $j$th subject. That is, if $y_{i}^{(j)}=1$, this means that $\mathbf{x}_{i}$ is categorized as a face from the $j$th subject.

On our problem setting, we should give two necessary remarks. One is that in our investigated face identification problems, almost all data have not been labeled before our system running. Only very small amount of samples are
annotated as the initialization. That is, most of $\{\mathbf{y}_{i}\}_{i=1}^n$ are unknown and needed to be completed in the learning process. In our system, a minority of them is manually annotated by the active users and a majority is pseudo-labeled in a self-paced manner. The other remark is that the data $\{\mathbf{x}_i\}_{i=1}^n$ might possibly been inputted into the system in an incremental way. This means that the data scale might be consistently growing.

Via the proposed mechanism of combining SPL and AL, our proposed ASPL model can adaptively handle both manually annotated and pseudo-labeled samples, and still progressively fit the consistently growing unlabeled data in such an incremental manner. The ASPL is formulated as follows: %
\begin{eqnarray}
\label{eq:obj}
\min_{\{\mathbf{w},\mathbf{b},\mathbf{v},\mathbf{y}_i \in \{-1, 1\}^m, i \notin \Omega^{\bm \lambda}\}} \sum_{j=1}^{m}
\frac{1}{2}\Vert \mathbf{w}^{(j)} \Vert_2^2 + \ \ \ \ \ \ \ \ \ \ \ \ \ \ \ \ \ \ \\  C \cdot
 L\left( \mathbf{w}^{(j)}, {b}^{(j)}, \mathcal{D},\mathbf{y}^{(j)}, \mathbf{v}^{(j)}\right)
 + \mathnormal{f}\left( \mathbf{v}^{(j)};{\lambda_j} \right) \nonumber \\
s.t.\quad \mathbf{v}\in {\bm \Psi} ^{\bm \lambda }, \nonumber
\end{eqnarray}
where $\mathbf{w=}\{\mathbf{w}^{(j)}\}_{j=1}^{m}\subset R^{d}$ and $\mathbf{%
b=\{}b^{(j)}\mathbf{\}}_{j=1}^{m}\subset R$ represent the weight and bias
parameters of the decision functions for all $m$ classifiers. $C (C>0)$ is the standard regularization parameter trading off the loss function and the margin, and we set $C=1$ in our experiments. $\mathbf{v}=\{[%
v_{1}^{(j)},v_{2}^{(j)},\cdots ,v_{n}^{(j)}]^{T}\}_{j=1}^m$ denotes the weight variables reflecting the training samples' importance, and $\lambda_j$ is a parameter (i.e. the pace age) for controlling the learning pace of the $j$th classifier. $\mathnormal{f}\left( \mathbf{v}^{(j)};\lambda_j\right) $ is the self-paced regularizer controlling the learning scheme.
We denote the index collection of all currently active annotated samples as $\Omega^{\bm \lambda}=\cup_{j=1}^m\{\Omega^{\lambda_j}\}$, where $\Omega^{\lambda_j}$ corresponds to the set of the $j$th subject with the pace age $\lambda_j$. Here $\Omega^{\bm \lambda}$ is introduced as a constraint on $\mathbf{y}_i$. ${\bm \Psi} ^{\bm \lambda}=\cap_{i=1}^{n} \{\Psi_{i}^{\bm \lambda}\}$ composes of the curriculum constraint of the model at the $m$ classifiers' pace age ${\bm \lambda}={\{\lambda_j\}}_{j=1}^m$. In particular, we specify two alternative types of the curriculum constraint for each sample $\mathbf{x}_{i}$, as:
\begin{itemize}
\item  $\Psi_{i}^{\bm \lambda} = [0,1]$ is for the pseudo-labeled sample, i.e., $i\notin \Omega ^{\bm \lambda}$. Then, its importance weights with respect to all the classifiers $\{v_{i}^{(j)}\}_{j=1}^m$ need to be learned in the SPL optimization.

\item $\Psi _{i}^{\bm \lambda}=\{1\}$ is for the sample annotated by the AL process, i.e., $\exists j \ \ s.t. \ \ i \in \Omega ^{\lambda_j}$. Thus, its importance weights are deterministically set during the model training, i.e., $v_{i}^{(j)}=1$.
\end{itemize}

Each type of the curriculums will be detailedly interpreted in Section \ref{sec:related_work}. Note that different from the previous SPL settings, this curriculum $\Psi _{i}^{\bm \lambda}$ can be dynamically changed with respect to all the pace ages  ${\bm \lambda}$ of $m$ classifiers. This conducts the superiority of our model, as we discuss in the end of this section.

{
We then define the loss function $L\left( \mathbf{w}^{(j)},{b}^{(j)}, \mathcal{D},\mathbf{y}^{(j)}, \mathbf{v}^{(j)}\right) $ on $\mathbf{x}$ as:%
\begin{equation}
\begin{split}
&L\left( \mathbf{w}^{(j)},{b}^{(j)}, \mathcal{D},\mathbf{y}^{(j)}, \mathbf{v}^{(j)}\right) \\
=& \sum_{i=1}^n v_i^{(j)} l\left( \mathbf{w}^{(j)},b^{(j)};\mathbf{x}_{i},y_{i}^{(j)}\right) \\
=& \sum_{i=1}^n v_i^{(j)} \left( 1-y_{i}^{(j)}(\mathbf{w}^{(j)T}\mathbf{x}_{i} + b^{(j)}) \right)_{+} \\
& \text{s.t.} \sum_{j=1}^m |y_i^{(j)} + 1 | \le 2, y_i^{(j)}\in \{-1, 1\}, i \notin \Omega^{\lambda},
\end{split}
\end{equation}
where $\left( 1-y_{i}^{(j)}(\mathbf{w}^{(j)T}\mathbf{x}_{i} + b^{(j)}) \right)_{+}$ is the hinge loss of $\mathbf{x}_{i}$ in the $j$th classifier. The cost term corresponds to the summarized loss of all classifiers, and the constraint term only allows two kinds of feasible solutions: i) for any $i$, there exists $y_{i}^{(j)}=1$ while for all other $y_i^{(k)}=-1$ for all $k \not=j$; ii) $y_{i}^{(j)} = -1$ for all $j=1, 2, \cdots, m$ (i.e., background or an unknown person class). These samples {$\mathbf{x}_i$} will be added into the unknown sample set $U$. It is easy to see that such constraint complies with real cases where a sample should be categorized into one pre-specified subject or not classified into any of the current subjects.}

Referring to the known alternative search strategy, we can then solve this
optimization problem. Specifically, the algorithm is designed by alternatively
updating the classifier parameters $\mathbf{w},\mathbf{b}$ via one-vs-all SVM, the sample importance weights $\mathbf{v}$ via the SPL, the pseudo-label $\mathbf{y}$ via reranking. Along with gradually increasing pace parameter ${\bm \lambda}$, the optimization updates: i) the curriculum constraint ${\bm \Psi}^{\bm \lambda}$ via AL and ii) the feature representation via CNN fine-tuning.  In the following we introduce the details of these optimization steps, and give their physical interpretations. The correspondence of this algorithm to the practical implementation of the ASPL system will also be discussed in the end.

\textbf{Initialization}: As introduced in the framework, we initialize our system running by using pre-trained CNN to extract feature representations of all samples $\{\mathbf{x}_{i}\}_{i=1}^{n}$. Set an initial $m$ classifiers' pace parameter set ${\bm \lambda}=\{\lambda_j\}_{j=1}^{m}$. Initialize the curriculum constraint ${\bm \Psi}^{\bm \lambda}$ with currently user annotated samples $\Omega^{\bm \lambda}$ and corresponding $\{\mathbf{y}^{(j)}\}_{j=1}^{m}$ and $\mathbf{v}$.

\textbf{Classifier Updating}: This step aims to update the classifier parameters $\{\mathbf{w}^{(j)},b^{(j)}\}_{j=1}^m$ by one-vs-all SVM. Fixing $\{\{\mathbf{x}_i\}_{i=1}^n,\mathbf{v},\{\mathbf{y}_{i}\}_{i=1}^{n}, {\bm \Psi}^{\bm \lambda}\}$, the original ASPL model Eqn.~(\ref{eq:obj}) can be simplified into the following form:

\begin{equation}
\min_{\mathbf{w},\mathbf{b}} \sum_{j=1}^{m} \frac{1}{2} \Vert \mathbf{w}^{(j)} \Vert_2^2 + C\sum_{i=1}^n v_{i}^{(j)}l\left( \mathbf{w}^{(j)},b^{(j)};\mathbf{x}_{i},y_{i}^{(j)}\right), \nonumber
\end{equation}
which can be equivalently reformulated as solving the following independent sub-optimization problems for each classifier $j=1,2,\cdots, m$:
\begin{equation}
\min_{\mathbf{w}^{(j)},b^{(j)}} \frac{1}{2} \Vert \mathbf{w}^{(j)} \Vert_2^2 +  C\sum_{i =1}^n v_{i}^{(j)}l\left( \mathbf{w}
^{(j)},b^{(j)};\mathbf{x}_{i},y_{i}^{(j)}\right).
\end{equation}
This is a standard one-vs-all SVM model with weights by taking one-class sample as positive while all others as negative. Specifically, when the weights $v_{i}^{(j)}$ are only of values $\{0,1\}$, it corresponds to a simplified SVM model under sampled instances with $v_{i}^{(j)}=1$; otherwise when $v_i^{j}$ sets values from $[0,1]$, it corresponds to the weighted SVM model. And both of them can be readily solved by many off-the-shelf efficient solvers. Thus, this step can be interpreted as implementing one-vs-all SVM over instances manually annotated from the AL and self-annotated from the SPL.

\textbf{High-confidence Sample Labeling}: This step aims to assign pseudo-labels $\mathbf{y}$ and corresponding important weights $\mathbf{v}$ to the top-ranked samples of high confidences.

We start by employing the SPL to rank the unlabeled samples according to their importance weights $\mathbf{v}$. Under fixed $\{\mathbf{w,b},\{\mathbf{x}_i\}_{i=1}^n, \{\mathbf{y}_i\}_{i=1}^n, {\bm \Psi}^{\bm \lambda}\}$, our ASPL model in Eqn.~(\ref{eq:obj}) can be simplified to optimize $\mathbf{v}$ as:%

\begin{equation}
\begin{gathered}
\min_{\mathbf{v}\in \left[0,1\right]}\sum_{j=1}^m C \sum_{i=1}^n v_{i}^{(j)}l\left(
\mathbf{w}^{(j)},{b}^{(j)};\mathbf{x}_{i},y_{i}^{(j)}\right) +%
\mathnormal{f}\left( \mathbf{v}^{(j)};\lambda_j \right), \\
s.t.\quad \mathbf{v}\in {\bm \Psi} ^{\bm \lambda }.
\end{gathered}
\end{equation}

This problem then degenerates to a standard SPL problem as in Eqn.(\ref{eq:spcl_obj}). Since both the self-paced regularizer $\mathnormal{f}( \mathbf{v}^{(j)};\lambda_j)$ and the curriculum constraint ${\bm \Psi ^{\bm \lambda}}$ is convex (with respect to $\mathbf{v}$), various existing convex optimization techniques, like the gradient-based or interior-point methods, can be used for solving it. Note that we have multiple choices for the self-paced regularizer, as those built in \cite{spl_reranking}\cite{spld}. All of them comply with three axiomic conditions required for a self-paced regularizer, as defined in Section \ref{sec:related_work}.




Based on the second axiomatic condition for self-paced regularizer, any of the above $\mathnormal{f}( \mathbf{v}^{(j)};\lambda_j)$ inclines to conduct larger weights on high-confidence (i.e., easy) samples with less loss values while vice versa, which evidently facilitates the model with the ``learning from easy to hard'' insight.
In all our experiments, we utilize the linear soft weighting regularizer
due to its relatively easy implementation and well adaptability to complex scenarios. This regularizer penalizes the sample weights linearly in terms of the loss. Specifically, we have

\begin{equation}\label{eq:regularizer}
f({\bf v}^{(j)}, \lambda_j) = \lambda_j ( \frac{1}{2} \Vert {\bf v}^{(j)} \Vert_2^2 - \sum_{i=1}^n v_i^{(j)}),
\end{equation}
where $\lambda_j > 0$. Eqn. (\ref{eq:regularizer}) is convex with respect to $\mathbf{v}^{(j)}$, and we can thus search for its global optimum by computing the partial gradient equals. Considering $v_{i}^{(j)} \in [0,1]$, we deduce the analytical solution for the linear soft weighting, as,

\begin{equation}\label{eq:v}
v_i^{(j)} = \left\{
\begin{array}{c}
 - \frac{C \ell_{ij}}{\lambda_j} + 1, \\
0,%
\end{array}
\begin{array}{c}
C \ell_{ij} < \lambda_j \\
\text{otherwise,}%
\end{array}%
\right.
\end{equation}
where $\ell_{ij} = l\left( \mathbf{w}^{(j)},b^{(j)};\mathbf{x}_i,y_{i}^{(j)}\right)
$ is the loss of $\mathbf{x}_{i}$ in the $j$th classifier. Note that the deducing way to Eqn. (\ref{eq:v}) is similar with in \cite{spl_reranking}, but our resulting solution is different since our ASPL model in Eqn. (\ref{eq:obj}) is new.

{
After obtaining the weight $\mathbf{v}$ for all unlabeled samples ($i \notin \Omega^{\lambda}$) according to the optimized $\mathbf{v}^{(j)}$ in a descending order. Then we consider the samples with larger important weight than others are high confidences. We form these samples into high-confidence sample set $\mathcal{S}$ and assign them pseudo-labels: Fixing \{$\mathbf{w,b},\{\mathbf{x}_i\}_{i=1}^n, {\bm \Psi
^{\bm \lambda}}, \mathbf{v}$\}, we optimize $\mathbf{y}_{i}$ of Eqn.~(\ref{eq:obj}) which corresponds to solve:%
\begin{equation}
\label{eq:y}
\begin{split}
&\min_{\mathbf{y}_{i} \in \{-1, 1\}^m, i \in \mathcal{S}} \sum_{i=1}^{n} \sum_{j=1}^{m}  v_{i}^{(j)} \ell_{ij} \\
&\text{s.t.}, \sum_{j=1}^m |y_i^{(j)} + 1 | \le 2.
\end{split}
\end{equation}
where $\mathbf{v}_i$ is fixed and can be treated as constant. When $\mathbf{x}_i$ belongs to a certain person class,  Eqn.~(\ref{eq:y}) has an optimum, which can be exactly extracted by the Theorem~\ref{theorem1}. The proof is specified in the supplementary material.

Denote those $j$s that satisfy $\mathbf{w}^{(j)T}\mathbf{x}_{i}+b^{(j)}\neq 0$ and $v_{i}^{(j)}\in (0,1]$ as a set $M$ and set all $y_{i}^{(j)}=-1$ for others in default \footnote{$v_i^{(j)}=0$ actually implies that the $i$-th sample is with low-confidence to be annotated as the $j$-th class, and thus it is natural to pseudo-label it as a negative sample for the $j$-th class. $w^Tx+b=0$ implies that a sample is located in the classification boundary of the $j$-th class, and thus it is also a low-confidence $j$-class sample and thus we directly annotate it as negative. Actually, for these samples, pseudo-label them as positive or negative will not affect the value of the objective function of Eq.~(\ref{eq:y}). We tend to annotate these low-confidence samples as negative since due to the constraint of Eq. (\ref{eq:y}) (at most one positive class one sample is allowed to be annotated), this will not influence selecting a more rational positive class for each sample.}. The solution of Eqn.~(\ref{eq:y}) for $y_{i}^{(j)},~j\in M$ can be obtained by the following theorem.

\begin{theorem}
\label{theorem1}
\text{}

(a) If ~~$\forall j \in M$,~~$\mathbf{w}^{(j)T}\mathbf{x}_{i}+b^{(j)}<0$,  Eqn.~(\ref{eq:y}) has a solution:
\begin{equation*}
y_{i}^{(j)}={-1},~~~j=1,...,m;
\end{equation*}

(b) When ~~~$\forall j \in M$~except $j=j^{\ast }$, $\mathbf{w}^{(j)T}\mathbf{x}_{i}+b^{(j)}<0$, i.e., $v_{i}^{(j^*)}\ell_{ij^*}>0$, then Eqn.~(\ref{eq:y}) has a solution:
\begin{equation*}
y_{i}^{(j)}=\left\{
\begin{array}{c}
-1,\text{ }j\neq j^{\ast } \\
~~~1,\text{ }j=j^{\ast }
\end{array}
;\right.
\end{equation*}

(c) Otherwise, Eqn.~(\ref{eq:y}) has a solution:
\begin{equation*}
y_{i}^{(j)}=\left\{
\begin{array}{c}
-1,\text{ }j\neq j^{\ast } \\
~~~1,\text{ }j=j^{\ast }
\end{array}
,\right.
\end{equation*}
where
\begin{equation}
\begin{split}
j^{\ast }=\arg \underset{1\leq j\leq m}{\min }v_{i}^{(j)} \left( \ell_{ij}  -
 \left( 1+( \mathbf{w}^{^{(j)}T}\mathbf{x}_{i}+b^{(j)})
\right) _{+} \right). 
\end{split}
\end{equation}
\end{theorem}

Actually, only those high-confidence samples with positive weights, as calculated in the last updating step for $\mathbf{v}$, are meaningful for the solution. This implies the physical interpretation for this optimization step: we iteratively find the high-confidence samples based on the current classifier, and then enforce pseudo-labels $\mathbf{y}_{i}$ on those top-ranked high-confidence ones ($i \in \mathcal{S}$). This is exactly the mechanism underlying a reranking technique~\cite{spl_reranking}.
}

The above optimization process can be understood as the self-learning manner of a student. The student tends to pick up most high-confident samples,
which imply easier aspects and faithful knowledge underlying data, to learn,
under the regularization of the pre-designed curriculum ${\bm \Psi}^{\bm \lambda }$. Such regularization inclines to rectify his/her learning process so as to avoid him/her stuck into a unexpected overfitting point.

{

\textbf{Low-confidence Sample Annotating}:
After pseudo-labeling high-confidence samples in such a self-paced uncertainty modeling, we employ AL fashion to update the curriculum constraint $\mathbf{ \Psi}^\mathbf{ \lambda}$ in the model by supplementing more informative curriculums based on human knowledge. The AL process aims to select most low-confidence unlabeled samples and to annotate them as either positive or negative by requesting user annotation. Our selection criteria are based on the classical uncertainty-based strategy \cite{lewis1994sequential,tong2002support}. Specifically, given the current classifiers, we randomly collected a number of randomly unlabeled samples, which are usually located in low-confidence area near the classification boundaries.

{ \textit{1) Annotated Sample Verifying:}
Considering the user annotation may contain outliers (incorrectly annotated samples), we introduce a verification step to correct the wrongly annotated samples. Assuming that labeled samples with lower prediction scores from the current classifiers have higher probability of being incorrectly labeled, we propose to ask the active user to verify their annotations on these samples. Specifically, in this step we first employ the current classifiers to obtain the prediction scores of all the annotated samples. Then we re-rank them and select Top-$L$ ones with lowest prediction scores and ask the user to verify these selected samples, i.e., double-checking them. We can set L as a small number ($L$ = 5 in our experiments), since we do believe the chance of human making mistakes is low. In sum, we improve the robustness of the AL process by further validating Top-L most uncertain samples with the user. In this way, we can reduce the effects of accumulated human annotation errors and enable the classifier to be trained in a robust manner.}

{ \textit{2) Low-confidence Definition}:
When we utilize the current classifiers ($m$ classifiers for discriminating $m$ object categories) to predict the label of unlabeled samples, those predicted as more than two positive labels (i.e., predicted as the corresponding object category) actually represent these samples making the current classifiers ambiguous. We thus adopt them as so called "low-confident" samples and require active user to manually annotate them. Actually, in this step, other "low-confidence" criterion can be utilized. We employed this simple strategy just due to its intuitive rationality and efficiency.
}

After users perform manual annotation, we update the $\mathbf{ \Psi}^\mathbf{ \lambda}$ by additionally incorporating those newly annotated sample set $\phi$ into the current curriculum $\mathbf{ \Psi}^\mathbf{ \lambda}$. For each annotated sample, our AL process includes the following two operations: i) Set its curriculum constraint, i.e., $\{\Psi_{i}^\mathbf{ \lambda}\}_{i \in \phi} =\{1\}$; ii) Update its labels $\{\mathbf{y}_{i}\}_{i \in \phi}$ and add its index into the set of currently annotated samples $\Omega^\mathbf{ \lambda }$.
Such specified curriculum still complies with the axiomic conditions for the curriculum constraint as defined in~\cite{spcl}. For those annotated samples, the corresponding $\Psi_{i}^\mathbf{ \lambda }=\{1\}$ with expectation value $1$ over the whole set, while for others $\Psi_{i}^\mathbf{ \lambda}=[0,1]$ with expectation value $1/2$. Thus the more informative samples still have a larger expectation than the others. Also, it is easy to see $\mathbf{ \Psi}^\mathbf{ \lambda}$\ is non-empty and convex. It thus complies traditional curriculum understanding.

\textbf{New Class Handling}: After the AL process, if active user annotates the selected unlabeled samples with $u$ unseen person classes, new classifiers for these unseen classes are needed to be initialized without affecting the existed classifiers. Moreover, there is another difficulty that the samples of the new class are not enough for classifier training. Thanks to the proposed ASPL framework, we can employ the following four steps to address above mentioned issues.
\begin{enumerate}
\item For each of these new class samples, search all the unlabeled samples and pick out its $K$-nearest neighbors from the unseen class set $U$ in the feature space;
\item Require active user to annotate these selected neighbors to enrich the positive samples for these new person classes;
\item Initialize and update $\{\mathbf{w}^{(j)}, b^{(j)}, \mathbf{v}^{(j)},\mathbf{y}^{(j)}, \lambda_{j}\}_{j=m+1}^{m+u}$ for these new person classes according to above mentioned iteration process of \{\emph{initialization, classifier updating, high-confidence sample labeling, low-confidence sample annotating}\}.
\end{enumerate}

}

\begin{algorithm}
\caption{The sketch of ASPL framework} \label{alg:alg_overview}
\begin{algorithmic}[1]
\REQUIRE Input dataset $\{\mathbf{x}_{i}\}_{i=1}^{n}$
\ENSURE Model parameters $\mathbf{w}$, $\mathbf{b}$
\STATE Use pre-trained CNN to extract feature representations of $\{\mathbf{x}_{i}\}_{i=1}^{n}$. Initialize multiple annotated samples into the curriculum ${\bm \Psi}^{\bm \lambda}$ and corresponding $\{\mathbf{y}_{i}\}_{i=1}^{n}$ and $\mathbf{v}$. Set an initial pace parameter ${\bm \lambda}=\{\lambda^0\}^m$.
\\

\textbf{while} not converged do
\STATE \ \ \ Update $\mathbf{w,b}$ by one-vs-all SVM
\STATE \ \ \ Update $\mathbf{v}$ by the SPL via Eqn.~(\ref{eq:v})
\STATE \ \ \ {Pseudo-label high-confidence samples $\{\mathbf{y}_i\}_{i \in \mathcal{S}}$ by the reranking via Eqn.~(\ref{eq:y})}
\STATE \ \ \ {Update the unclear class set $U$}
{ \STATE \ \ \ Verify the annotated samples by AL.}
\STATE \ \ \ {Update low-confidence samples $\{\mathbf{y}_i, \Psi_i^\lambda\}_{i \in \mathcal{\phi}}$ by the AL} \\
{\ \ \ \textbf{if} $u$ unseen classes have labeled, \\
\ \ \ \ \ \ \ \ Handle $u$ new classes via the steps in Sect.~\ref{sec:alg}\\
\ \ \ \ \ \ \ \ Go to the step 2 \\
\ \ \ \ \textbf{end if}
}

\STATE \ \ \ In every $T$ iterations: {
\begin{itemize}
\item Update $\{\mathbf{x}_{i}\}_{i=1}^{n}$ through fine-tuning CNN
\item Update ${\bm \lambda}$ according to Eqn.~(\ref{eq:lambda})
\end{itemize} }
\STATE
\textbf{end while}
\RETURN $\mathbf{w,b}$;
\end{algorithmic}
\end{algorithm}

This step corresponds to the instructor's role in human education, which aims to guide a student to involve more informative curriculums in learning. Different from the previous fixed curriculum setting in SPL throughout the learning process, here the curriculum is dynamically updated based on the self-paced learned knowledge of the model. Such an improvement better simulates the general learning process of a good student. With the learned knowledge of a student increasing, his/her instructor should vary the curriculum settings imposed on him from more in the early stage to less in later. This learning manner evidently should conduct a better learning effect which can well adapt the personal information of the student.

\textbf{Feature Representation Updating}:
After several of the SPL and AL updating iterations of \{$\mathbf{w,b},\{\mathbf{y}_{i}\}_{i=1}^{n},\mathbf{v},{\bm \Psi}^{\bm \lambda}$\}, we now aim to update the feature representation $\{\mathbf{x}_{i}\}_{i=1}^{n}$ through finetuning the pretrained
CNN by inputting all manually labeled samples from the AL and self-annotated
ones from the SPL. These samples tend to deliver data knowledge into the
network and improve the representation of the training samples. A better feature
representation is thus expected to be extracted from this ameliorated CNN.

This learning process simulates the updating of the knowledge structure of a human brain after a period of domain learning. Such updating tends to facilitate a person grasp more effective features to represent newly coming samples from certain domain and make him/her with a better learning performance. In our experiments, we generally conduct the CNN feature fine-tuning after around $50$ rounds of the SPL and AL updating, and the learning rate is set as 0.001 for all layers.

\textbf{Pace Parameter Updating}: We utilize a heuristic strategy to update pace parameters $\{\lambda_j\}_{j=1}^m$ for $m$ classifiers in our implementation.

After multiple iterations of the ASPL, we specifically set the pace parameter $\lambda_j$ for each individual classifier, and utilize a heuristic strategy in our implementation for parameter updating. For the $t$th iteration, we compute the pace parameter for optimizing Eqn.~(\ref{eq:obj}) by :
\begin{equation}
\label{eq:lambda}
\lambda_j^t = \left\{
\begin{array}{c}
\text{ }\lambda^0,\text{ } \ \ \ \ \ \ \ \ \ \ \ \ \ \ \ \ \ \ \ \ t = 0 \\
\text{} \lambda_j^{(t-1)} + \alpha * \eta_j^t, \ \ \ \ \ \ \text{ } 1 \le t \le  \tau\\
\text{} \lambda_j^{(t-1)}, \ \ \ \ \ \ \ \ \ \ \ \ \ \ \ \ \ \ t > \tau,
\end{array}
\right.
\end{equation}
where $\eta_j^t$ is the average accuracy of the $j$-th classifier in the current iteration, and $\alpha$ is a parameter which controls the pace increasing rate. In our experiments, we empirically set $\{\lambda^0, \alpha \}=\{0.2, 0.08\}$. Note that the pace parameters ${\bm \lambda}$ should be stopped when all training samples are with $\mathbf{v}=\{1\}$. Thus, we introduce an empirical threshold $\tau$ constraining that ${\bm \lambda}$ is only updated in early iterations, i.e., $t \le \tau$. $\tau$ is set as 12 in our experiments.

The entire algorithm can then be summarized into Algorithm 1. It is easy to see that this solving strategy for the ASPL model finely accords with the pipeline of our framework.

{\emph{Convergence Discussion}: As illustrated in Algorithm~\ref{alg:alg_overview}, the ASPL algorithm alternatively updates variables including: the classifier parameters $w$, $b$ (by weighted SVM), the pseudo-labels $y$ (closed-form solution by Theorem 1), the importance weight $v$ (by SPL), and low-confidence sample annotations $\phi$ (by AL). For the first three parameters, these updates are calculated by a global optimum obtained from a sub-problem of the original model, and thus the objective function can be guaranteed to be decreased. However, just as other existing AL techniques, human efforts are involved in the loop of the AL stage, and thus the objective function cannot be guaranteed to be monotonically decreased in this step. However, just as shows in Sect.~\ref{sec:experiment}, as the learning processing, the model tends to be more and more mature, and the labor of AL tends to be less and less in the later learning stage. Thus with gradually less involvement of the AL calculation in our algorithm, the monotonic decrease of the objective function in iteration tends to be promised, and thus our algorithm tends to be convergent.}

\subsection{Relationship with Other SPL/AL Models}

It is easy to see that the proposed ASPL model extends the previous AL/SPL models and includes all of them as special cases. When we fix the curriculum and feature representations and only update other parameters, it degenerates to the traditional SPL models by rationally setting the self-paced regularizer. When we fix the SPL parameters, feature representations and do not involve pseudo-labels in learning, the model degenerates the a general AL learning regime. The amelioration to both SPL and AL is expected to bring benefits to both regimes. On one hand, introducing more high-confidence samples in the self-paced fashion is helpful to reduce the burden of user annotations, particularly when the classifier becomes reliable at later learning iterations. On the other hand, the low confidence samples selected by active user annotations tends to make our approach workable with less initial labeled samples than existing self-paced learning algorithms. All these benefits are comprehensively substantiated by our experiments.

%
%



\begin{table}[!ht]
\begin{center}
\caption{The summarization of datasets we used.}
\label{table:dataset_setting}
\begin{tabular}{cccc} \hline
		Dataset & \# images & \# persons & \# images/person \\
	\hline
		CACD    & 56,138    & 500        &  	79$\sim$306	\\
CASIA-WebFace-Sub & 181,901 & 925 & 100$\sim$804 \\
	\hline
\end{tabular}
\end{center}
\vspace{-10pt}
\end{table}
\section{Experiments}
\label{sec:experiment}

In this section, we first introduce the datasets and implementation setting, and then discuss the experimental results and comparisons with other existing approaches.

\subsection{Datasets and Setting}
\label{ssec:datasets_settings}

We adopt two public datasets in our experiments, the Cross-Age Celebrity Dataset (CACD)~\cite{cross-age} { and CASIA-WebFace-Sub dataset~\cite{webface}}.

CACD is a large-scale and challenging dataset for evaluating face recognition and retrieval, and it contains a batch of images of $2,000$ celebrities collected from Internet, which are varying in age, pose, illumination, and occlusion. And only a subset of $200$ celebrities are manually annotated by Chen et al. \cite{cross-age}. For better convincing evaluation, we augment this subset by extra labeling $300$ individuals and obtain a set of $56,138$ images in total.

CASIA-WebFace dataset~\cite{webface} is a large scale face recognition dataset with 10,575 subjects/persons and 494,414 images. CASIA-WebFace is extremely challenging for its images are all collected from Internet with different view points and light illumination under different scenes. Though the total person/subject number of CASIA-WebFace dataset is very large, the sample number for each person, varying from 3 to 804, is heavily unbalanced. For those persons who has very few samples (say below 100), the experiment analysis is not able to be performed. Hence, we select a subset of the CASIA-WebFace dataset by discarding its persons with less than 100 samples to form the CASIA-WebFace-Sub dataset. The CASIA-WebFace-Sub dataset has 181,901 images with 925 persons inside. The detailed information of above mentioned datasets is summarized in Table~\ref{table:dataset_setting}.

\textbf{\textit{Experiment setting. }} We detect the facial points using the method proposed in \cite{SDM} and align the faces based on the eye locations. The experiments on both of the datasets are conducted as the following steps. We first randomly select $80\%$ images of each individual to form the unlabeled training set, and the rest samples are used for testing, according to the setting in the existing active learning method \cite{con_AL}. Then, we randomly annotate $n$ samples of each person in the training set to initialize the classifier. To get rid of the influence of randomness, we average the results over $5$ times of execution with different sample selections. All of the experiments are conducted on a common desktop PC with i7 3.4GHz CPU and a NVIDIA Titan X GPU.

On the two above mentioned datasets, we evaluate the performance of incremental face identification in two aspects: the recognition accuracy and user annotation amount in the incremental learning process. The recognition accuracy is defined as the rank-one rate for face identification.
We compare our ASPL framework with several existing active learning algorithms and baseline methods under the same setting: i) CPAL (Convex Programming based Active Learning) \cite{con_AL}: Annotate a few samples in each step based on prediction uncertainty and sample diversity; ii) CCAL (Confidence-based Active Learning via SVMs) \cite{tong2002support}: Select only one sample having lowest prediction confidence; iii) AL\_RAND: Randomly select unlabeled samples to be annotated during the training phase. This method discards all active learning techniques and can be considered as the lower bound, and iv) AL\_ALL: All unlabeled samples are annotated for training the classifier. This method can be regarded as the upper bound (best performance the classifier can achieve). {For fair comparison, all of these methods utilize the same feature representation as ours in the beginning. As the training iteration increase, active user annotation is employed to those selected most informative and representative samples. Then, CNN fine-tuning is also exploited to improve the feature extractor for ASPL, CPAL, CCAL, AL\_RAND, AL\_ALL.}

\textbf{\textit{Details of CNN implementation. }} The architecture of AlexNet~\cite{alexnet} is utilized in our all experiments. Thanks to the well pre-training, the CNN updating is only implemented few times during ASPL iteration in all our experiments, each only containing no more than 5 CNN updating steps. We generally conducted CNN steps after around 5 rounds of the SPL and AL updating, and the learning rate is set as 0.001 for all layers. Equal importance is imposed between the previous training examples and the newly labeled examples, and CNN is updated using the stochastic gradient decent methods with the momentum 0.9 and weight decay 0.0005.


\begin{figure*}[!htb]
\centering
\centering
\includegraphics[width=0.7\textwidth]{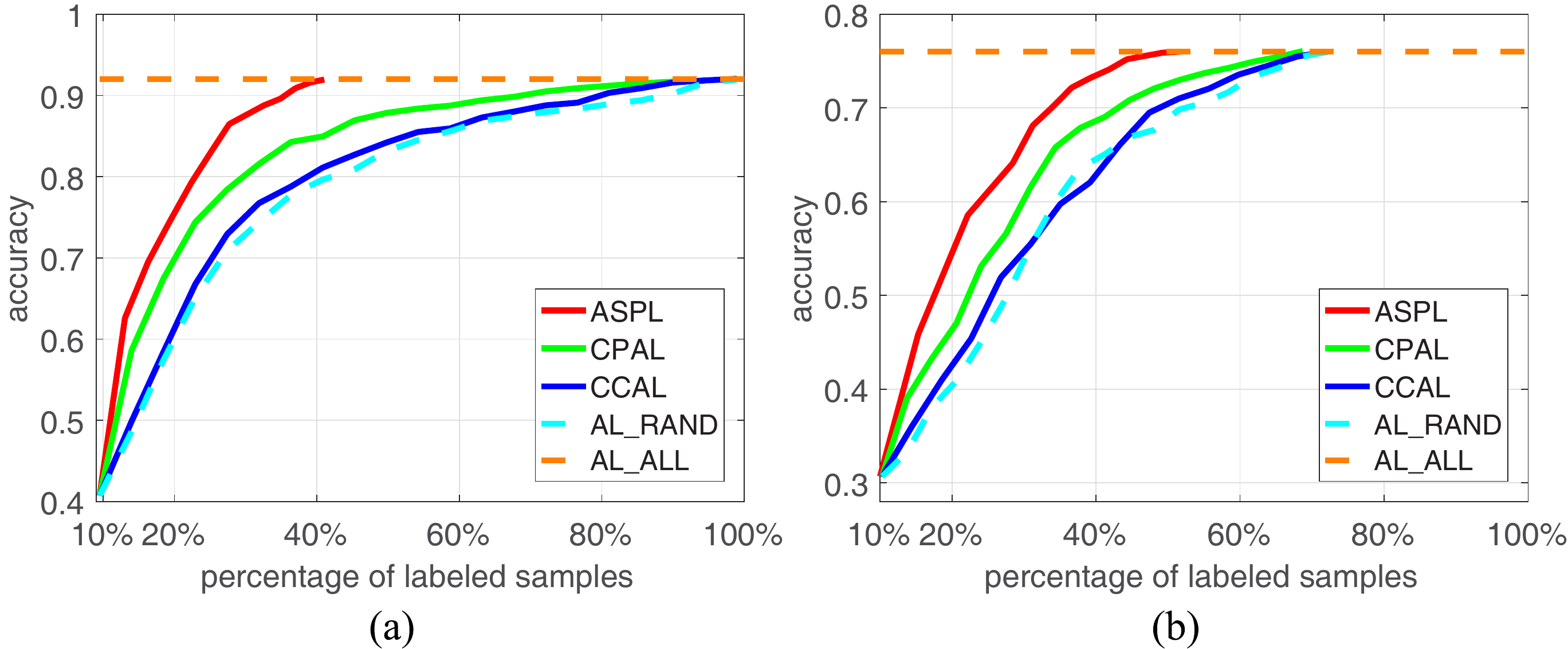}
\vspace{-5pt}
\label{fig:component_al_ai}
\caption{Results on (a) CACD and (b) CASIA-WebFace-Sub datasets. The vertical axes represent the recognition accuracy and the horizontal axes represent the percentage of annotated samples of the whole set.}
\vspace{-10pt}
\label{fig:cmp_alg_all}
\end{figure*}

\begin{figure}[ht]
\centering
\includegraphics[width=0.9\columnwidth]{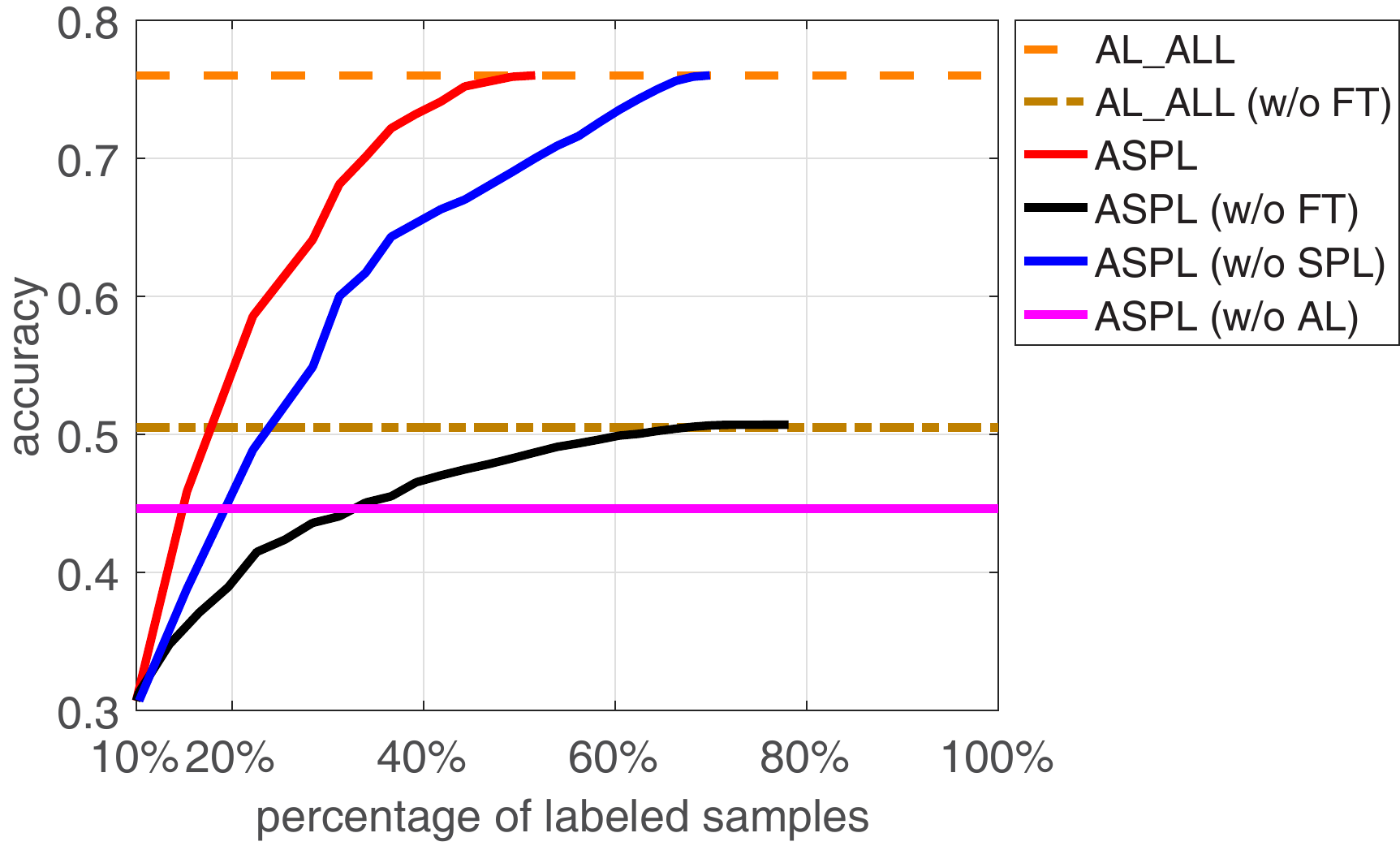}
\caption{Accuracies with the increase of annotated samples of different variants of our framework, using CASIA-Webface-Sub dataset. }
\vspace{-10pt}
\label{fig:component_al}
\end{figure}

\subsection{Experimental Comparisons}
{
The results on the two datasets are reported in Fig.~\ref{fig:cmp_alg_all}(a) and Fig.~\ref{fig:cmp_alg_all}(b), respectively, where we can observe how the recognition accuracy changes with increasingly incorporating more unlabeled samples. In CACD dataset, to achieve the same recognition accuracy, ASPL model requires few annotation of the unlabeled data. On the other hand, ASPL outperforms the competing methods in accuracy when the same amount annotations. ASPL can still have a superior performance as the iteration goes on. The similar results and phenomena can be discovered in CASIA-WebFace-Sub dataset. As one can see that, ASPL only requires about 40\% and 45\% annotations to achieve the-state-of-art performance on CACD and CASIA-WebFace-Sub dataset, respectively. While the compared methods AL\_RAND, CCAL and CPAL all requires about 81\% and 65\%, respectively. Hence, our ASPL can performs as well as the AL\_ALL with minimal annotations.

Note that the performances of RAND and CCAL are relatively close, and the similar results were reported in \cite{con_AL}. According to the explanation in \cite{con_AL}, this comes from the fact that many samples have low prediction confidences and distribute not densely in the feature space. Thus, the randomizing sample selection achieves similar results compared to CCAL. }

\subsection{Component Analysis}
\label{ssec:exper_ca}
{
To further analyze how different components contribute to performance, we implement several variants of our framework: i) ASPL (w/o FT): allowing both active and self-paced sample selection during learning while disabling the CNN fine-tuning, i.e., the feature extractor is kept the same as the iteration goes on for training; ii) ASPL (w/o SPL): discarding high-confidence sample pseudo-labeling via self-paced learning; iii) ASPL (w/o AL): ignoring low confidence samples for active user annotation; iv) AL\_ALL: fine-tuning the CNN and train classifiers with all the labels of the training samples and v) AL\_ALL (w/o FT): training classifiers with all the labels of the training samples without fine-tuning. Moreover, the full version of our proposed model is denoted as ASPL, which allows the convolutional nets to be fine-tuned during the training process. We further evaluate the ASPL variants in the following aspects.

\textbf{\textit{Contribution of different ASPL components.}}
Using AL\_ALL and AL\_ALL (w/o FT) as the baselines, we gradually add the AL, SPL and fine-tuning components to ASPL. These experiments are executed on the CASIA-Webface dataset. Fig. \ref{fig:component_al} illustrates the accuracy obtained using ASPL, ASPL (w/o FT), ASPL (w/o AL) and ASPL (w/o SPL). One can observe that any of the three components is useful in improving the recognition accuracy. Especially, the additional SPL component can significantly improve the recognition accuracy and reduce the number of annotation samples by automatically exploiting the majority of high-confidence samples for feature learning.

We also observe that the CNN feature fine-tuning can dramatically improve the recognition accuracy in the early steps. This is mainly because the information gain (i.e., individual appearance diversity) deceases with progressively introducing new samples to the neural nets. 

\textbf{\textit{Analysis on initial samples. }} In SPL \cite{spl_kumar}, classifier is first trained using the initial samples. With the current classifier, easy samples are preferred to be selected in the early training steps, and thus it is expected that the performance of SPL heavily relies on the initial samples. Fortunately, by incorporating with active learning, ASPL can evidently alleviate this problem. To verify this, we compare the performance of ASPL and SPL on 20 randomly selected individuals of CASIA-Webface-Sub dataset. The result is shown in Fig.~\ref{fig:diff_init}. Given the same initialized feature representations, we also conduct the experiments to analyze the performance vs different initial portions to be handled by AL on this dataset. The results are illustrated in Fig.~\ref{fig:init_portion}.

As one can see from Fig.~\ref{fig:diff_init}, with different initial samples, ASPL reaches similar/stable results as the training continues, while SPL still varies a lot. This result indicates that the AL component is effective in handling the  poor initialization. Fig.~\ref{fig:init_portion} illustrates that though poor performance is obtained at the beginning, the performance of our model increases during the training process. In summary, our model is insensitive to the diversity and quantity of initial samples.

\begin{figure}[!htb]
\begin{center}
\includegraphics[width=0.65\columnwidth]{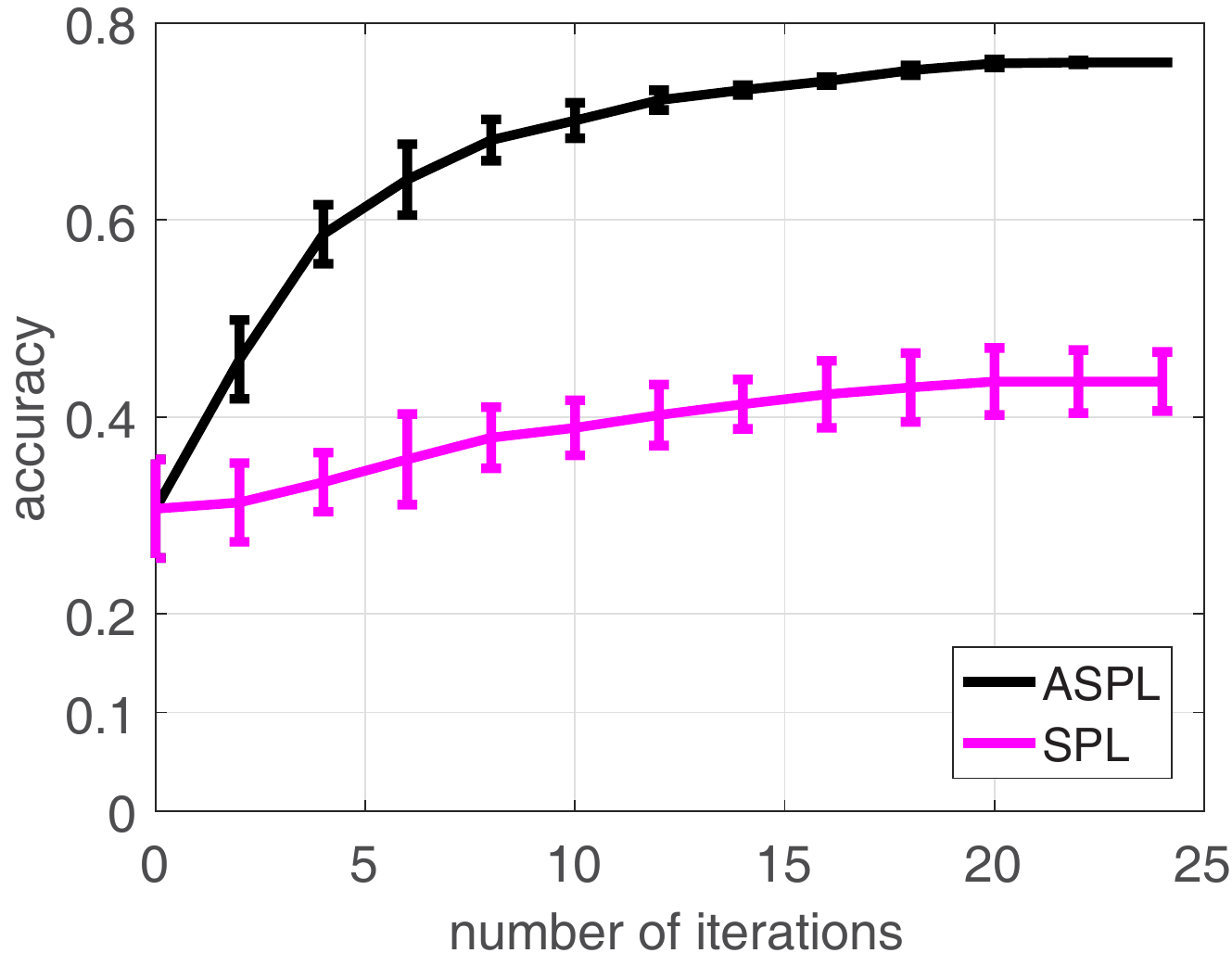}
\vspace{-5pt}
\caption{The accuracy and standard deviation of ASPL and SPL on the CASIA-Webface-Sub dataset.}
\vspace{-10pt}
\label{fig:diff_init}
\end{center}
\end{figure}

\begin{figure}[!htb]
\begin{center}
\includegraphics[width=0.65 \columnwidth]{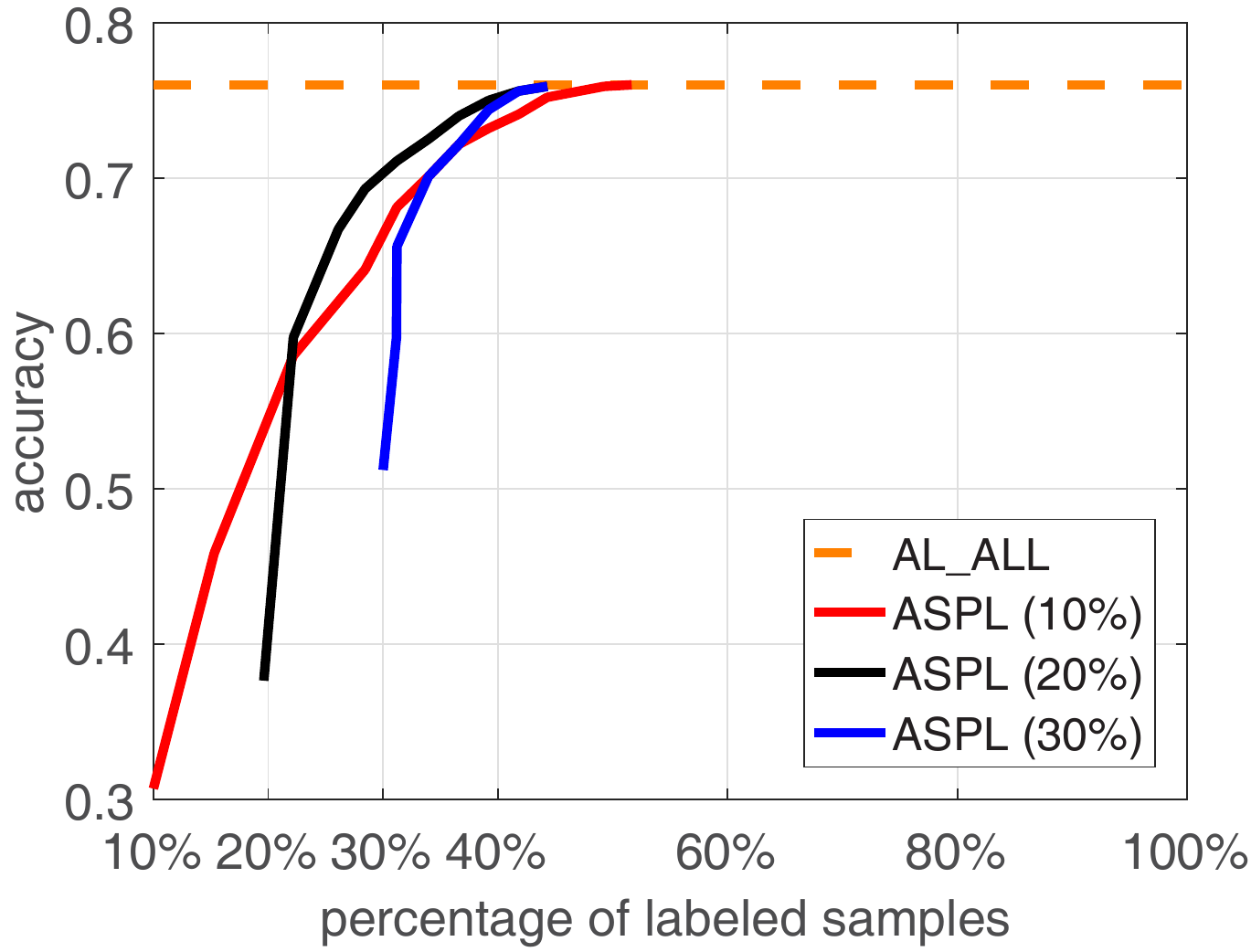}
\vspace{-5pt}
\caption{The comparison of different number of initial samples and the further required annotation ported of the AL process on the CASIA-Webface-Sub dataset. For fair comparison, these methods share the same feature representation as initialization.}
\vspace{-10pt}
\label{fig:init_portion}
\end{center}
\end{figure}

\begin{table}[!ht]
\begin{center}
\caption{The performance comparison of whether handling unseen new classes or not on the CASIA-WebFace-Sub dataset. ASPL (ALL) denotes the ASPL version of no unseen classes.}
\label{table:newclass}
\begin{tabular}{ccccc} \hline
		\# Class Number  & 300 & 600 & 925 &	\\
	\hline
		ASPL (ALL)  &  88.3\%  &  81.0\%  &  76.0\%  \\
		ASPL  &  88.3\%  & 81.6\%  &  76.0\%  \\
	\hline
\end{tabular}
\end{center}
\vspace{-15pt}
\end{table}

\begin{table}[htpb]
\begin{center}
\caption{The error rates of the pseudo-labels assigned by SPL on high-confidence samples.}
\label{table:error_rate}
\begin{tabular}{ c c c c c c } \hline
		\# iteration 	 	  & 5 	& 10 	& 	15 & 20	  & 25 	\\
	\hline
		ASPL (w/o FT)   &  8.2\%  &  6.9\%  &  5.1\% & 5.0\% & 4.9\% \\
		ASPL             &  4.5\%  &  4.1\%  &  3.4\% & 3.3\% & 3.3\%  \\
	\hline
\end{tabular}
\end{center}
\end{table}

\begin{figure}[htpb]
\centering
	\centering
	\includegraphics[width=0.63\columnwidth]{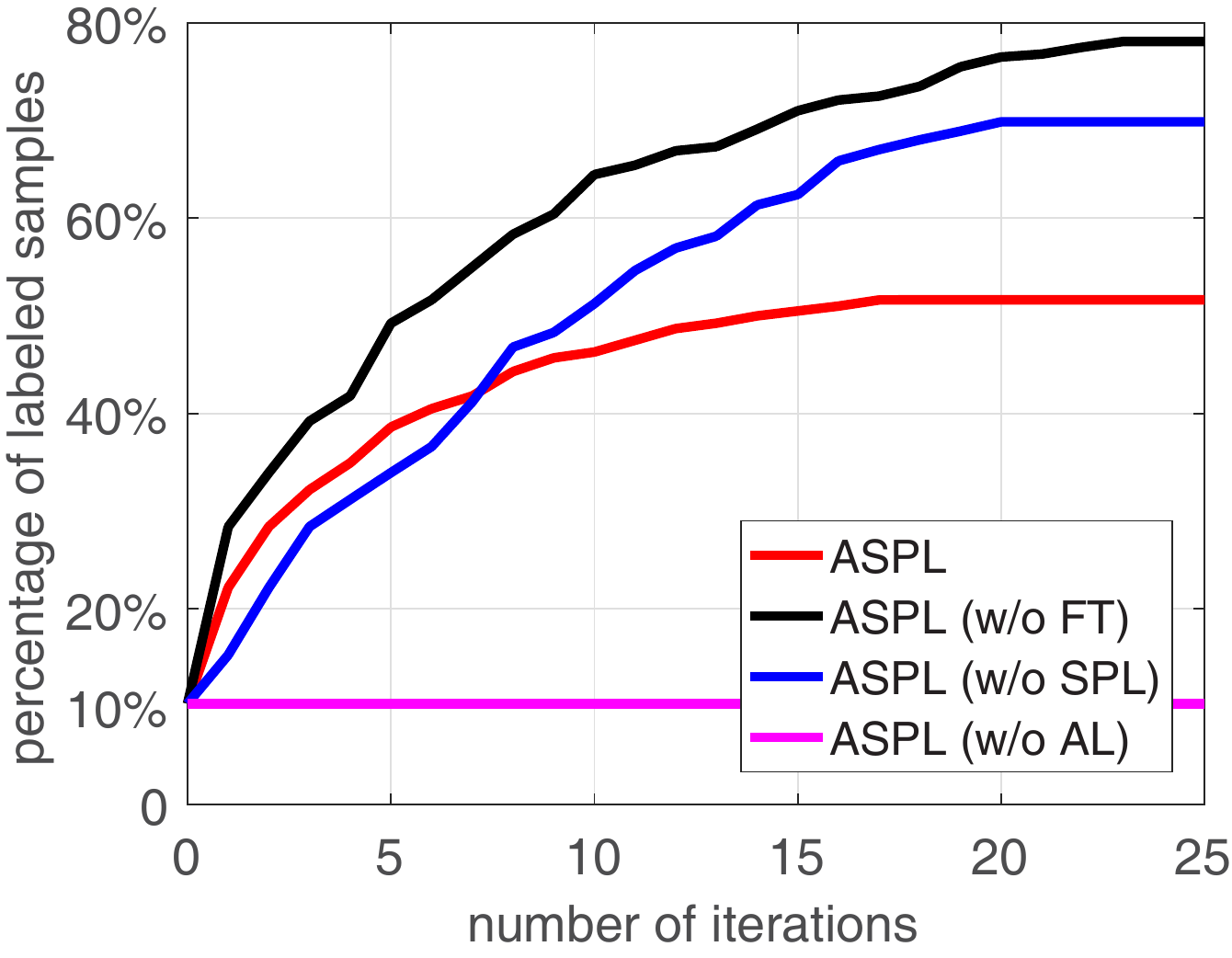}
	\vspace{-5pt}
\caption{The comparison of different number of initial samples and the further required annotation ported of the AL process on the CASIA-WebFace-Sub 
dataset.}
\vspace{-15pt}
\label{fig:AL_needs}
\end{figure}

\textbf{\textit{Performance with new classes.}}
To justify the effectiveness of our ASPL for handling unseen new classes, we conduct the following experiment on the CASIA-WebFace-Sub dataset: We compare the performance of incrementally giving some classes (our ASPL) and directly giving all person classes. Specifically, given all person classes,  we initialize all the classifiers at the beginning of the training and optimize them without handling unseen new classes. We denote this variant as ASPL (ALL). The experimental result is illustrated in Table~\ref{table:newclass} and shows that our proposed ASPL can handle unseen new classes effectively without substantially performance drop or even with slightly better performance, compared with the all classes given version ASPL (ALL).

\textbf{\textit{Annotation required for large scale dataset.}}
To demonstrate that our ASPL can be adopted under large scale scenario, we analyze the training phase of ASPL on the large scale CASIA-WebFace-Sub dataset. As illustrated in Fig.~\ref{fig:AL_needs}, the x-axis denotes the number of training iterations and the y-axis denotes the amount of required user annotation. The curve in Fig.~\ref{fig:AL_needs} demonstrates that our proposed ASPL model requires relatively larger annotations when the training iteration number is small. As the training continues, the amount required annotations began to be reduced due to the gradually mature model incrementally ameliorated in the learning process. This observation indicates that the burden of user annotations would be indeed relieved when the classifier becomes reliable at the later learning stage of the proposed ASPL method. Moreover, as illustrated in Table \ref{table:error_rate}, with the increase of user annotations over time, ASPL can automatically assign more reliable pseudo-labels to the unlabeled samples selected in the self-paced way.

\begin{figure*}[!htb]
\begin{center}
\includegraphics[width=0.65\textwidth]{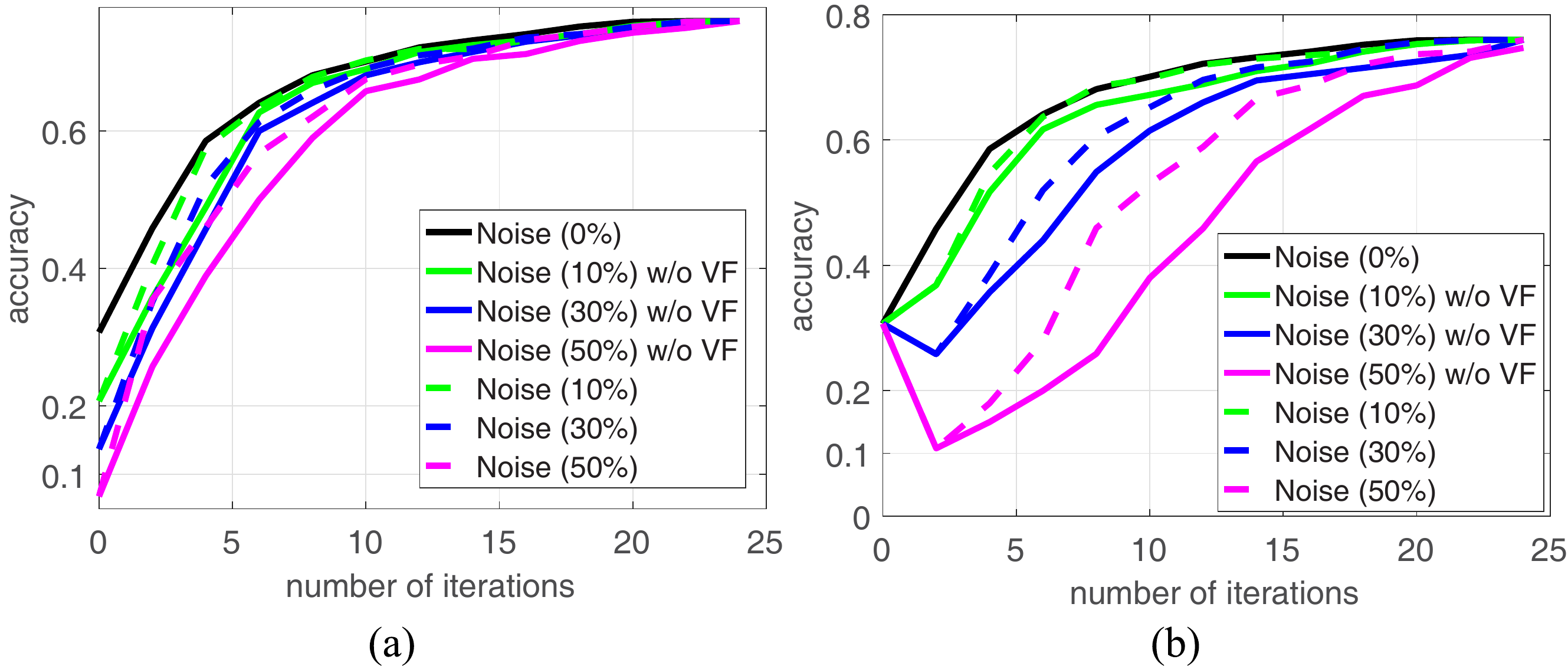}
\vspace{-5pt}
\caption{Robust analysis of ASPL under two types of noisy samples. (a) Using different number of noisy samples as the initial annotation. (b) Adding different number of noisy samples at the 10-th step (denoted by the black spots).}
\vspace{-20pt}
\label{fig:robust}
\end{center}
\end{figure*}

\textbf{\textit{Robustness analysis. }}
{ We further analyze the robustness of ASPL when noisy images are deliberately included in two experiments. (i) Ex-1: $a$ ($a = 0\%, 10\%, 30\%, 50\%$) noisy images are added to the initial samples for each individual. (ii) Ex-2: noise-free initials are used, but $b$ ($b = 0\%, 10\%, 30\%, 50\%$) importers are deliberately annotated during the training process. These experiments are conducted on the CASIA-Webface-Sub dataset. To validate the effectiveness of the proposed annotated sample verifying step, we disable the verifying step and denote these modification as ``Noise w/o VF''.

Fig. \ref{fig:robust}(a) shows the result of Ex-1, where ASPL is initialized with different number of noisy images. In early steps of the iteration, noisy data have huge adverse effect on test accuracy. Along with the increase of iteration number, the genuine data gradually dominate the results.
Fig. \ref{fig:robust}(b) illustrates the result of Ex-2, where noisy images are added to the labeled training set the 2-th step of iteration. We can see that a sharp decline in the recognition accuracy. However, with the evolving of ASPL training, similar accuracy as compared with that got on the original clean data can be obtained when the number of iterations increases. As one can comparing ``Noise (10/30/50\%)'' with ``Noise (10/30/50\%) w/o VF'' from Fig.~\ref{fig:robust}(a), with the verifying step, ASPL can recover from noisy images in a slightly fast way. This justifies the effectiveness of the proposed annotated sample verifying step.
}

\section{Conclusions}
\label{sec:conclusion}

In this paper, we have introduced, first, an effective framework to solve incremental face identification, which build classifiers by progressively annotating and selecting unlabeled samples in an active self-paced way, and second, a theoretical interpretation of the proposed framework pipeline from the perspective of optimization. Third, we evaluate our approach on challenging scenarios and show very promising results.

In the future, we will extend the system to support several video-based vision applications, which require large amount of user annotations. The proposed framework provides a rational realization to this task by automatically distinguishing high-confidence samples, which can be easily and faithfully recognized by computers in a self-paced way, and low-confidence ones, which can be discovered by requesting user annotation.

\appendix
\textbf{Proof for Theorem 1}

{
Our aim is to solve the following optimization problem:
\begin{equation}
\label{eq:y}
\min_{\mathbf{y}_{i}\in \{-1,1\}^{m},i\in \mathcal{S}}\sum_{j=1}^{m}v_{i}^{(j)}\ell _{ij},\ \ \text{s.t.}~\sum_{j=1}^{m}|y_{i}^{(j)}+1|\leq 2,
\end{equation}
where $\ell_{ij}=l\left( \mathbf{w}^{(j)},b^{(j)};\mathbf{x}_{i},y_{i}^{(j)}\right) $ is the hinge loss of $\mathbf{x}_{i}$ in the $j$th classifier. Specifically, we define the hinge loss as:
\begin{equation*}
l\left( \mathbf{w}^{(j)},b^{(j)};\mathbf{x}_{i},y_{i}^{(j)}\right) =\left(
1-y_{i}^{(j)}\left( \mathbf{w}^{^{(j)}T}\mathbf{x}_{i}+b^{(j)}\right)
\right) _{+}.
\end{equation*}
The constraint term
\begin{equation}
\sum_{j=1}^{m}|y_{i}^{(j)}+1|\leq 2 \label{eq:indicator}
\end{equation}
dominates two cases of $\mathbf{y}_{i}$ can be for all $m$ classifiers: (i) all items of $\mathbf{y}_{i}$ are \emph{all negative}, i.e., $\{y_{i}^{(j)}\}_{j=1}^{m}=\{-1\}$. In this case, the input region proposal $\mathbf{x}_{i}$ is assumed to be the background by $m$ classifiers in the current optimization. (ii) In all items of $\mathbf{y}_{i}$, \emph{one is positive and all others are negative}. In this case, $\mathbf{x}_{i}$ is categorized into a certain object class.

\vspace{15pt}

Before giving the solution of Eqn.~(\ref{eq:y}), we first introduce the two necessary lemmas as follows:

\textbf{Lemma 1}: The solution of
\begin{equation}
\min_{y_{i}^{(j)}\in \{-1,1\},i\in \mathcal{S}}\text{ }\ell _{ij},~j=1,...,m
\label{eq: 1}
\end{equation}

is:
\begin{equation*}
y_{i}^{(j)}=\left\{
\begin{array}{c}
-1,\text{ if }\mathbf{w}^{(j)T}\mathbf{x}_{i}+b^{(j)}<0 \\
1,\text{ if }\mathbf{w}^{(j)T}\mathbf{x}_{i}+b^{(j)}>0 \\
1\text{ or }-1,\text{ if }\mathbf{w}^{(j)T}\mathbf{x}_{i}+b^{(j)}=0%
\end{array}%
\right. .
\end{equation*}

\begin{proof} We discuss the solution in three cases:

(i) When $\mathbf{w}^{(j)T}\mathbf{x}_{i}+b^{(j)}<0$, it is easy to see that
\begin{equation*}
\left( 1-\left( \mathbf{w}^{^{(j)}T}\mathbf{x}_{i}+b^{(j)}\right) \right)
_{+}>\left( 1+\left( \mathbf{w}^{^{(j)}T}\mathbf{x}_{i}+b^{(j)}\right)
\right) _{+}.
\end{equation*}

Thus the global solution of Eqn.~(\ref{eq: 1}) is $y_{i}^{(j)}=-1$.

(ii) When $\mathbf{w}^{(j)T}\mathbf{x}_{i}+b^{(j)}>0$, similar to (i), one can easily prove that $y_{i}^{(j)}=1$ is the global solution in this case.

(iii) When $\mathbf{w}^{(j)T}\mathbf{x}_{i}+b^{(j)}=0$, whether $y_{i}^{(j)}$ = $1$ or $-1$, $l_{ij}$ will have the same value $1$. Thus both $y_{i}^{(j)}=1$ and $y_{i}^{(j)}=-1$ are the global solution of Eqn.~(\ref{eq: 1}).
\end{proof}

\textbf{Lemma 2}: The solution of
\begin{equation}
\min_{\mathbf{y}_{i}\in \{-1,1\}^{m},i\in \mathcal{S}}v_{i}^{(j)}\ell _{ij},~ j = 1, ..., m,
\label{eq: 2}
\end{equation}

is:
\begin{equation*}
y_{i}^{(j)}=\left\{
\begin{array}{c}
-1,\text{ if }\mathbf{w}^{(j)T}\mathbf{x}_{i}+b^{(j)}<0\text{ and }v_{i}^{(j)}\in (0,1] \\
~~~1,\text{ if }\mathbf{w}^{(j)T}\mathbf{x}_{i}+b^{(j)}>0\text{ and }v_{i}^{(j)}\in (0,1] \\
1\text{ or }-1,\text{ if }\mathbf{w}^{(j)T}\mathbf{x}_{i}+b^{(j)}=0\text{
or }v_{i}^{(j)}=0
\end{array}
.\right.
\end{equation*}

\begin{proof}
For $v_{i}^{(j)}\in (0,1]$, since $v_{i}^{(j)}$ is a positive constant, the solution of Eqn.~(\ref{eq: 2}) is the same as that of Eqn.~(\ref{eq: 1}). While if $v_{i}^{(j)}=0$, for both $y_{i}^{(j)}=1$ or $y_{i}^{(j)}=-1$,~~$l_{ij}$ will have the same value $0$. The conclusion is thus evident.
\end{proof}

As one can easily see from \textbf{Lemma 2}, when $\mathbf{w}^{(j)T}\mathbf{x}_{i}+b^{(j)}=0$ or $v_{i}^{(j)}=0$, the optimal $y_{i}^{(j)}$ for Eqn.~(\ref{eq: 2}) can be either $+1$ or $-1$. Thus in all components $v_{i}^{(j)}\ell_{ij}$ of Eqn.~(\ref{eq:y}) with $\mathbf{w}^{(j)T}\mathbf{x}_{i}+b^{(j)}=0$ or $v_{i}^{(j)}=0$, we can easily assume that the corresponding solution is $y_{i}^{(j)}=-1$, i.e., $|y_{i}^{(j)}+1|=0$, which will not affect the soundness and final values of the optimal solution of Eqn.~(\ref{eq:y}). 

Denote those $j$s that satisfy $\mathbf{w}^{(j)T}\mathbf{x}_{i}+b^{(j)}\neq 0$ and $v_{i}^{(j)}\in (0,1]$ as a set $M$ and set all $y_{i}^{(j)}=-1$ for others in default. The solution of Eqn.~(\ref{eq:y}) for $y_{i}^{(j)},~j\in M$ can be obtained by the following theorem.

\begin{theorem}
\text{}

(a) If ~~$\forall j \in M$,~~$\mathbf{w}^{(j)T}\mathbf{x}_{i}+b^{(j)}<0$,  Eqn.~(\ref{eq:y}) has a solution:
\begin{equation*}
y_{i}^{(j)}={-1},~~~j=1,...,m;
\end{equation*}

(b) When ~~~$\forall j \in M$~except $j=j^{\ast }$, $\mathbf{w}^{(j)T}\mathbf{x}_{i}+b^{(j)}<0$, i.e., $v_{i}^{(j^*)}\ell_{ij^*}>0$, then Eqn.~(\ref{eq:y}) has a solution:
\begin{equation*}
y_{i}^{(j)}=\left\{
\begin{array}{c}
-1,\text{ }j\neq j^{\ast } \\
~~~1,\text{ }j=j^{\ast }
\end{array}
;\right.
\end{equation*}

(c) Otherwise, Eqn.~(\ref{eq:y}) has a solution:
\begin{equation*}
y_{i}^{(j)}=\left\{
\begin{array}{c}
-1,\text{ }j\neq j^{\ast } \\
~~~1,\text{ }j=j^{\ast }
\end{array}
,\right.
\end{equation*}
where
\begin{equation}
j^{\ast }=\arg \underset{1\leq j\leq m}{\min }v_{i}^{(j)}\left( \ell_{ij} -\left( 1+\left( \mathbf{w}^{^{(j)}T}\mathbf{x}_{i}+b^{(j)}\right)
\right) _{+}\right).  \label{eq: 3}
\end{equation}
\end{theorem}

\begin{proof} 
In the cases $(a)$ and $(b)$, it is easy to see that the provided $y_{i}^{(j)} $ is actually the solution of the unconstraint problem of Eqn.~(\ref{eq:y}). Since the solution complies with the constraint, this solution is also the one of the constrained one.

In the case $(c)$, there are more than two samples with positive confidence scores, i.e., $\mathbf{w}^{(j)T}\mathbf{x}_{i}+b^{(j)}>0$. In this case, it is impossible that the final solution is
\begin{equation*}
y_{i}^{(j)}=-1\text{, }j=1,...,m,
\end{equation*}since if we let $y_{i}^{(j)}=1$ for any one sample satisfying $\mathbf{w}^{(j)T}\mathbf{x}_{i}+b^{(j)}>0$, the objective function will have a decrease value with respect to $v_{i}^{(j)} > 0$. 
\begin{equation*}
v_{i}^{(j)}\left( \left( 1+\left( \mathbf{w}^{^{(j)}T}\mathbf{x}
_{i}+b^{(j)}\right) \right) _{+} -\ell_{ij}\right) >0.
\end{equation*}

Then there will be a unique $j^{\ast }$ where the final solution should have $y_{i}^{(j^{\ast })}=1$. We only need to pick up the one at which the objective of Eqn.~(\ref{eq:y}) attains the minimal value.

Assume $\mathbf{y}_{i}^{'}\in\{-1,1\}^{m}$ with
\begin{equation*}
y_{i}^{(j)}=\left\{
\begin{array}{c}
-1,\text{ }j\neq j^{\prime } \\
~~~1,\text{ }j=j^{\prime }
\end{array}
~~~.\right.
\end{equation*}
The objective of Eqn.~(\ref{eq:y}) is then
\begin{equation*}
\begin{aligned}
&F(j^{\prime })=\sum_{j=1}^{m}v_{i}^{(j)}\left( 1-y_{i}^{(j)}\left( \mathbf{w}^{^{(j)}T}\mathbf{x}_{i}+b^{(j)}\right) \right) _{+} \\
&=v_{i}^{(j^{\prime
})}\ell_{ij^{\prime}}+\sum_{j\neq j^{\prime }}\left(
1+\left( \mathbf{w}^{^{(j)}T}\mathbf{x}_{i}+b^{(j)}\right) \right) _{+}.
\end{aligned}
\end{equation*}
Then if we assume another $\mathbf{y}_{i}^{\ast }\in\{-1,1\}^{m}$ with
\begin{equation*}
y_{i}^{(j)}=\left\{
\begin{array}{c}
-1,\text{ }j\neq j^{\ast } \\
~~~1,\text{ }j=j^{\ast }%
\end{array}%
~~~,\right.
\end{equation*}
then we have that
\begin{eqnarray*}
&&F(j^{\prime })-F(j^{\ast }) \\
&=&v_{i}^{(j^{\prime })}\ell_{ij^{\prime}}+v_{i}^{(j^{\ast
})}\left( 1+\left( \mathbf{w}^{^{(j^{\ast })}T}\mathbf{x}_{i}+b^{(j^{\ast
})}\right) \right)_+  \\
&&-v_{i}^{(j^{\prime })}\left( 1+\left( \mathbf{w}^{^{(j^{\prime })}T}\mathbf{x}_{i}+b^{(j^{\prime })}\right) \right) _{+} - v_{i}^{(j^{\ast
})}\ell_{ij^{\ast}} \\
&=&\left(v_{i}^{(j^{\prime })}\ell_{ij^{\prime}}-v_{i}^{(j^{\prime
})}\left( 1+\left( \mathbf{w}^{^{(j^{\prime})}T}\mathbf{x}_{i}+b^{(j^{\prime
})}\right) \right)_+ \right)  \\
&&-\left( v_{i}^{(j^{\ast })}\ell_{ij^{\ast}}-v_{i}^{(j^{\ast
})}\left( 1+\left( \mathbf{w}^{^{(j^{\ast })}T}\mathbf{x}_{i}+b^{(j^{\ast
})}\right) \right)_{+} \right) \\
&=&v_{i}^{(j^{\prime })}\left(\ell_{ij^{\prime}}-\left( 1+\left( \mathbf{w}^{^{(j^{\prime})}T}\mathbf{x}_{i}+b^{(j^{\prime
})}\right) \right)_+ \right)  \\
&&-v_{i}^{(j^{\ast })}\left( \ell_{ij^{\ast}}-\left( 1+\left( \mathbf{w}^{^{(j^{\ast })}T}\mathbf{x}_{i}+b^{(j^{\ast
})}\right) \right)_{+} \right) \\
\end{eqnarray*}
If we choose $j^{\ast }$ as Eqn.~(\ref{eq: 3}), then it is easy to see that
\begin{equation*}
F(j^{\prime })-F(j^{\ast })\geq 0.
\end{equation*}
We thus can deduce that Eqn.~(\ref{eq:y}) is with a global solution:
\begin{equation*}
y_{i}^{(j)}=\left\{
\begin{array}{c}
-1,\text{ }j\neq j^{\ast } \\
~~~1,\text{ }j=j^{\ast }
\end{array}
~~~.\right.
\end{equation*}
The proof is completed.
\end{proof}
}


\bibliographystyle{ieee}
{\small
\bibliography{mybib}

\begin{thebibliography}{10}

\bibitem{celli2014automatic}
Fabio Celli, Elia Bruni, and Bruno Lepri,
\newblock ``Automatic personality and interaction style recognition from
  facebook profile pictures'',
\newblock in {\em ACM Conference on Multimedia}, 2014.

\bibitem{stone2010toward}
Zak Stone, Todd Zickler, and Trevor Darrell,
\newblock ``Toward large-scale face recognition using social network context'',
\newblock {\em Proceedings of the IEEE}, vol. 98, 2010.

\bibitem{sid15tcsvt}
Z.~Lei, D.~Yi, and S.~Z. Li,
\newblock ``Learning stacked image descriptor for face recognition'',
\newblock {\em IEEE Transactions on Circuits and Systems for Video Technology},
  vol. PP, no. 99, pp. 1--1, 2015.

\bibitem{pfr13pami}
S.~Liao, A.~K. Jain, and S.~Z. Li,
\newblock ``Partial face recognition: Alignment-free approach'',
\newblock {\em IEEE Transactions on Pattern Analysis and Machine Intelligence},
  vol. 35, no. 5, pp. 1193--1205, 2013.

\bibitem{tpr13cvpr}
D.~Yi, Z.~Lei, and S.~Z. Li,
\newblock ``Towards pose robust face recognition'',
\newblock in {\em Computer Vision and Pattern Recognition (CVPR), 2013 IEEE
  Conference on}, 2013, pp. 3539--3545.

\bibitem{hpen15cvpr}
Xiangyu Zhu, Z.~Lei, Junjie Yan, D.~Yi, and S.~Z. Li,
\newblock ``High-fidelity pose and expression normalization for face
  recognition in the wild'',
\newblock in {\em 2015 IEEE Conference on Computer Vision and Pattern
  Recognition (CVPR)}, 2015, pp. 787--796.

\bibitem{face_hybrid}
Yi~Sun, Xiaogang Wang, and Xiaoo Tang,
\newblock ``Hybrid deep learning for face verification'',
\newblock in {\em Proc. of IEEE International Conference on Computer Vision},
  2013.

\bibitem{AUSDL15ICCV}
X.~Wang, X.~Guo, and S.~Z. Li,
\newblock ``Adaptively unified semi-supervised dictionary learning with active
  points'',
\newblock in {\em 2015 IEEE International Conference on Computer Vision
  (ICCV)}, 2015, pp. 1787--1795.

\bibitem{hurtdata_pami15}
Yu-Feng Li and Zhi-Hua Zhou,
\newblock ``Towards making unlabeled data never hurt'',
\newblock {\em IEEE Trans. Pattern Anal. Mach. Intelligence}, vol. 37, no. 1,
  pp. 175--188, 2015.

\bibitem{iPCA}
Haitao Zhao et~al.,
\newblock ``A novel incremental principal component analysis and its
  application for face recognition'',
\newblock {\em SMC, IEEE Transactions on}, 2006.

\bibitem{kim2007incremental}
Tae-Kyun Kim, Kwan-Yee~Kenneth Wong, Bj{\"o}rn Stenger, Josef Kittler, and
  Roberto Cipolla,
\newblock ``Incremental linear discriminant analysis using sufficient spanning
  set approximations'',
\newblock in {\em Proc. of IEEE Conference on Computer Vision and Pattern
  Recognition}, 2007.

\bibitem{con_AL}
Elhamifar, Ehsan, Sapiro Guillermo, Yang Allen, and Sasrty~S Shankar,
\newblock ``A convex optimization framework for active learning'',
\newblock in {\em Proc. of IEEE International Conference on Computer Vision},
  2013.

\bibitem{keze_CEAL}
K.~Wang, D.~Zhang, Y.~Li, R.~Zhang, and L.~Lin,
\newblock ``Cost-effective active learning for deep image classification'',
\newblock {\em IEEE Transactions on Circuits and Systems for Video Technology},
  vol. PP, no. 99, pp. 1--1, 2016.

\bibitem{spcl}
Lu~Jiang, Deyu Meng, Qian Zhao, Shiguang Shan, and Alexander~G Hauptmann,
\newblock ``Self-paced curriculum learning'',
\newblock {\em Proc. of AAAI Conference on Artificial Intelligence}, 2015.

\bibitem{spld}
Lu~Jiang, Deyu Meng, Shoou-I Yu, Zhenzhong Lan, Shiguang Shan, and Alexander
  Hauptmann,
\newblock ``Self-paced learning with diversity'',
\newblock in {\em Proc. of Advances in Neural Information Processing Systems},
  2014.

\bibitem{spl_reranking}
Lu~Jiang, Deyu Meng, Teruko Mitamura, and Alexander~G Hauptmann,
\newblock ``Easy samples first: self-paced reranking for zero-example
  multimedia search'',
\newblock in {\em ACM Conference on Multimedia}, 2014.

\bibitem{curriculun_learning}
Yoshua Bengio, J{\'e}r{\^o}me Louradour, Ronan Collobert, and Jason Weston,
\newblock ``Curriculum learning'',
\newblock in {\em Proc. of IEEE International Conference on Machine Learning},
  2009.

\bibitem{spl_kumar}
M~Pawan Kumar et~al.,
\newblock ``Self-paced learning for latent variable models'',
\newblock in {\em Proc. of Advances in Neural Information Processing Systems},
  2010.

\bibitem{Hu_2015_ICCV_Workshops}
Guosheng Hu, Yongxin Yang, Dong Yi, Josef Kittler, William Christmas, Stan~Z.
  Li, and Timothy Hospedales,
\newblock ``When face recognition meets with deep learning: An evaluation of
  convolutional neural networks for face recognition'',
\newblock in {\em The IEEE International Conference on Computer Vision (ICCV)
  Workshops}, 2015.

\bibitem{lecun2010convolutional}
Yann LeCun, Koray Kavukcuoglu, and Cl{\'e}ment Farabet,
\newblock ``Convolutional networks and applications in vision'',
\newblock in {\em ISCAS}, 2010.

\bibitem{keze_dpl}
K.~Wang, L.~Lin, W.~Zuo, S.~Gu, and L.~Zhang,
\newblock ``Dictionary pair classifier driven convolutional neural networks for
  object detection'',
\newblock in {\em 2016 IEEE Conference on Computer Vision and Pattern
  Recognition (CVPR)}, June 2016, pp. 2138--2146.

\bibitem{ISed}
Chenqiang Gao, Deyu Meng, Wei Tong, Yi~Yang, Yang Cai, Haoquan Shen, Gaowen
  Liu, Shicheng Xu, and Alexander Hauptmann,
\newblock ``Interactive surveillance event detection through mid-level
  discriminative representation'',
\newblock in {\em ACM International Conference on Multimedia Retrieval}, 2014.

\bibitem{smith2002tutorial}
Lindsay~I Smith,
\newblock ``A tutorial on principal components analysis'',
\newblock {\em Cornell University, USA}, vol. 51, pp. 52, 2002.

\bibitem{ISVM}
Masayuki Karasuyama and Ichiro Takeuchi,
\newblock ``Multiple incremental decremental learning of support vector
  machines'',
\newblock in {\em Proc. of Advances in Neural Information Processing Systems},
  2009.

\bibitem{TNN-2006}
Nan-Ying Liang et~al.,
\newblock ``A fast and accurate online sequential learning algorithm for
  feedforward networks'',
\newblock {\em Neural Networks, IEEE Transactions on}, 2006.

\bibitem{ozawa2005incremental}
Seiichi Ozawa et~al.,
\newblock ``Incremental learning of feature space and classifier for face
  recognition'',
\newblock {\em Neural Networks}, vol. 18, 2005.

\bibitem{lewis1994sequential}
David~D Lewis and William~A Gale,
\newblock ``A sequential algorithm for training text classifiers'',
\newblock in {\em ACM SIGIR Conference}, 1994.

\bibitem{tong2002support}
Simon Tong and Daphne Koller,
\newblock ``Support vector machine active learning with applications to text
  classification'',
\newblock {\em The Journal of Machine Learning Research}, vol. 2, 2002.

\bibitem{mccallumzy1998employing}
Andrew~Kachites McCallumzy and Kamal Nigamy,
\newblock ``Employing em and pool-based active learning for text
  classification'',
\newblock in {\em Proc. of IEEE International Conference on Machine Learning},
  1998.

\bibitem{joshi2009multi}
Ajay~J Joshi, Fatih Porikli, and Nikolaos Papanikolopoulos,
\newblock ``Multi-class active learning for image classification'',
\newblock in {\em Proc. of IEEE Conference on Computer Vision and Pattern
  Recognition}, 2009.

\bibitem{kapoor2009faces}
Ashish Kapoor, Gang Hua, Amir Akbarzadeh, and Simon Baker,
\newblock ``Which faces to tag: Adding prior constraints into active
  learning'',
\newblock in {\em Proc. of IEEE International Conference on Computer Vision},
  2009.

\bibitem{kapoor2007active}
Ashish Kapoor, Kristen Grauman, Raquel Urtasun, and Trevor Darrell,
\newblock ``Active learning with gaussian processes for object
  categorization'',
\newblock in {\em Proc. of IEEE International Conference on Computer Vision},
  2007.

\bibitem{li2013adaptive}
Xin Li and Yuhong Guo,
\newblock ``Adaptive active learning for image classification'',
\newblock in {\em Proc. of IEEE Conference on Computer Vision and Pattern
  Recognition}, 2013.

\bibitem{brinker2003incorporating}
Klaus Brinker,
\newblock ``Incorporating diversity in active learning with support vector
  machines'',
\newblock in {\em Proc. of IEEE International Conference on Machine Learning},
  2003.

\bibitem{spmf}
Qian Zhao, Deyu Meng, Lu~Jiang, Qi~Xie, Zongben Xu, and Alexander~G Hauptmann,
\newblock ``Self-paced learning for matrix factorization'',
\newblock in {\em Proc. of AAAI Conference on Artificial Intelligence}, 2015.

\bibitem{spl_kumar_segment}
M~Pawan Kumar, Haithem Turki, Dan Preston, and Daphne Koller,
\newblock ``Learning specific-class segmentation from diverse data'',
\newblock in {\em Proc. of IEEE International Conference on Computer Vision},
  2011.

\bibitem{spl_Kristen}
Yong~Jae Lee and Kristen Grauman,
\newblock ``Learning the easy things first: Self-paced visual category
  discovery'',
\newblock in {\em Proc. of IEEE Conference on Computer Vision and Pattern
  Recognition}, 2011.

\bibitem{spl-tracking}
JS~Supancic and Deva Ramanan,
\newblock ``Self-paced learning for long-term tracking'',
\newblock in {\em Proc. of IEEE Conference on Computer Vision and Pattern
  Recognition}, 2013.

\bibitem{MED14}
S.~Yu~et al,
\newblock ``Cmu-informedia@ trecvid 2014 multimedia event detection'',
\newblock in {\em TRECVID Video Retrieval Evaluation Workshop}, 2014.

\bibitem{alexnet}
Alex Krizhevsky, Ilya Sutskever, and Geoffrey~E. Hinton,
\newblock ``Imagenet classification with deep convolutional neural networks'',
\newblock in {\em Advances in Neural Information Processing Systems 25}, pp.
  1097--1105. 2012.

\bibitem{vgg19}
K.~Simonyan and A.~Zisserman,
\newblock ``Very deep convolutional networks for large-scale image
  recognition'',
\newblock in {\em ICLR}, 2015.

\bibitem{cross-age}
Bor-Chun Chen, Chu-Song Chen, and Winston~H Hsu,
\newblock ``Cross-age reference coding for age-invariant face recognition and
  retrieval'',
\newblock in {\em ECCV}, 2014.

\bibitem{webface}
Dong Yi, Zhen Lei, Shengcai Liao, and Stan~Z. Li,
\newblock ``Learning face representation from scratch'',
\newblock {\em CoRR}, vol. abs/1411.7923, 2014.

\bibitem{SDM}
Xuehan Xiong and Fernando De~la Torre,
\newblock ``Supervised descent method and its applications to face alignment'',
\newblock in {\em Proc. of IEEE Conference on Computer Vision and Pattern
  Recognition}, 2013.

\end{thebibliography}
}

\vspace{-10pt}

\begin{IEEEbiography}[{\includegraphics[width=1in,height=1.25in,clip,keepaspectratio]{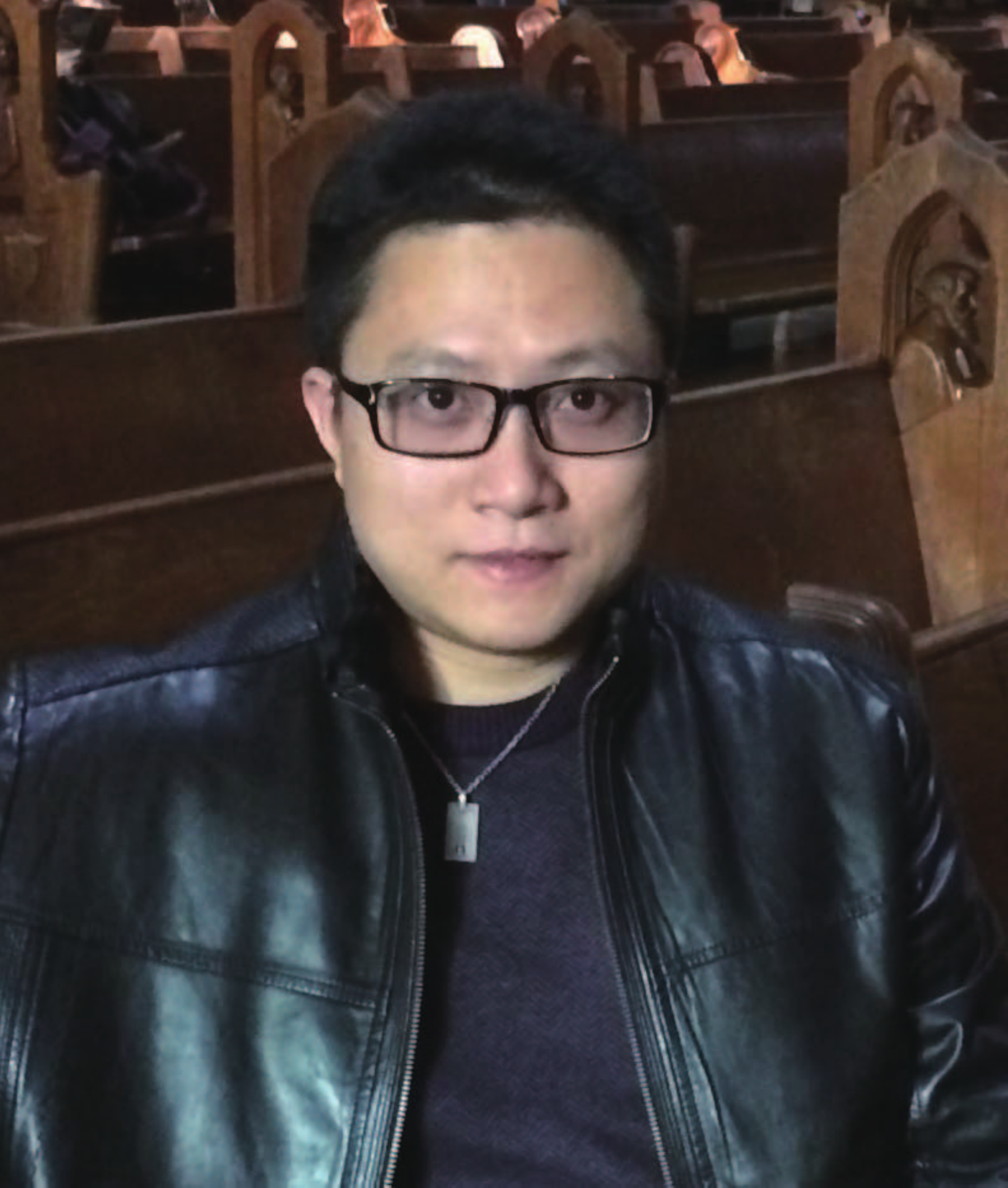}}]{Liang Lin} is a full Professor of Sun Yat-sen University. He is the Excellent Young Scientist of the National Natural Science Foundation of China. He received his B.S. and Ph.D. degrees from the Beijing Institute of Technology (BIT), Beijing, China, in 2003 and 2008, respectively, and was a joint Ph.D. student with the Department of Statistics, University of California, Los Angeles (UCLA). From 2008 to 2010, he was a Post-Doctoral Fellow at UCLA. From 2014 to 2015, as a senior visiting scholar he was with The Hong Kong Polytechnic University and The Chinese University of Hong Kong. His research interests include Computer Vision, Data Analysis and Mining, and Intelligent Robotic Systems, etc. Dr. Lin has authorized and co-authorized on more than 100 papers in top-tier academic journals and conferences. He has been serving as an associate editor of IEEE Trans. Human-Machine Systems. He was the recipient of the Best Paper Runners-Up Award in ACM NPAR 2010, Google Faculty Award in 2012, Best Student Paper Award in IEEE ICME 2014, and Hong Kong Scholars Award in 2014. More information can be found in his group website http://hcp.sysu.edu.cn
\end{IEEEbiography}

\begin{IEEEbiography}[{\includegraphics[width=1in,height=1.25in,clip,keepaspectratio]{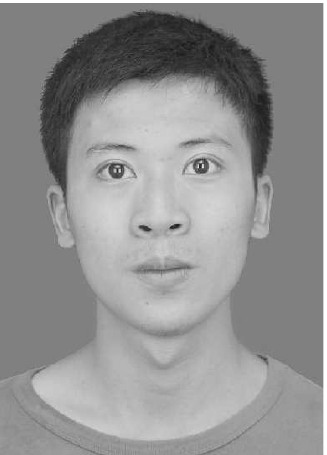}}]{Keze Wang} received his B.S. degree in software engineering from Sun Yat-Sen University, Guangzhou, China, in 2012. He is currently pursuing the Ph.D. degree in computer science and technology at Sun Yat-Sen University, advised by Professor Liang Lin. His current research interests include computer vision and machine learning.
\end{IEEEbiography}
\begin{IEEEbiography}[{\includegraphics[width=1in,height=1.25in,clip,keepaspectratio]{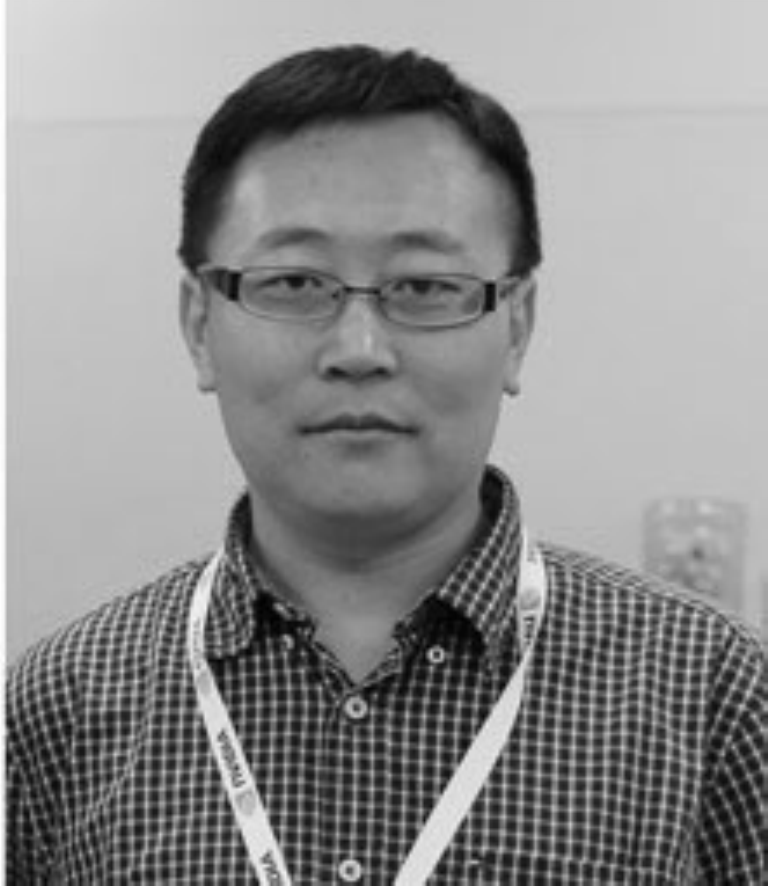}}]{Deyu Meng} Deyu Meng received the B.Sc., M.Sc., and Ph.D. degrees in 2001, 2004, and 2008, respectively, from Xi’an Jiaotong University, Xi'an, China.
He is currently a Professor with the Institute for Information and System Sciences, School of Mathematics and Statistics, Xi’an Jiaotong University. From 2012 to 2014, he took his two-year sabbatical leave in Carnegie Mellon University. His current research interests include self-paced learning, noise modeling, and tensor sparsity.
\end{IEEEbiography}
\begin{IEEEbiography}[{\includegraphics[width=1in,height=1.25in,clip,keepaspectratio]{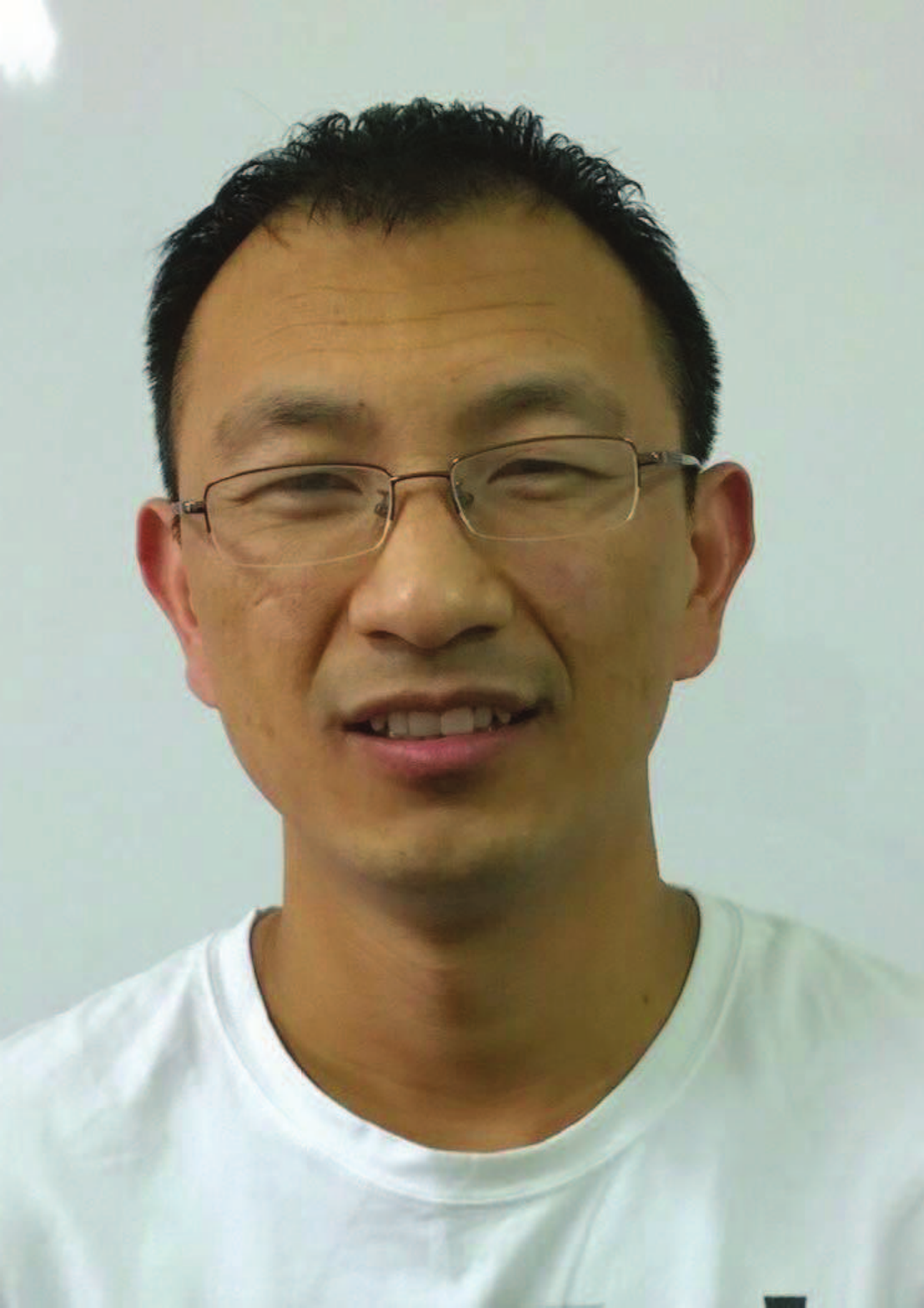}}]{Wangmeng Zuo} (M'09, SM'14) received the Ph.D. degree in computer application technology from the Harbin Institute of Technology, Harbin, China, in 2007. From July 2004 to December 2004, from November 2005 to August 2006, and from July 2007 to February 2008, he was a Research Assistant at the Department of Computing, Hong Kong Polytechnic University, Hong Kong. From August 2009 to February 2010, he was a Visiting Professor in Microsoft Research Asia. He is currently a Professor in the School of Computer Science and Technology, Harbin Institute of Technology. His current research interests include image enhancement and restoration, visual tracking, weakly supervised learning, and image classification. Dr. Zuo is an Associate Editor of the IET Biometrics.
\end{IEEEbiography}

\begin{IEEEbiography}[{\includegraphics[width=1in,height=1.25in,clip,keepaspectratio]{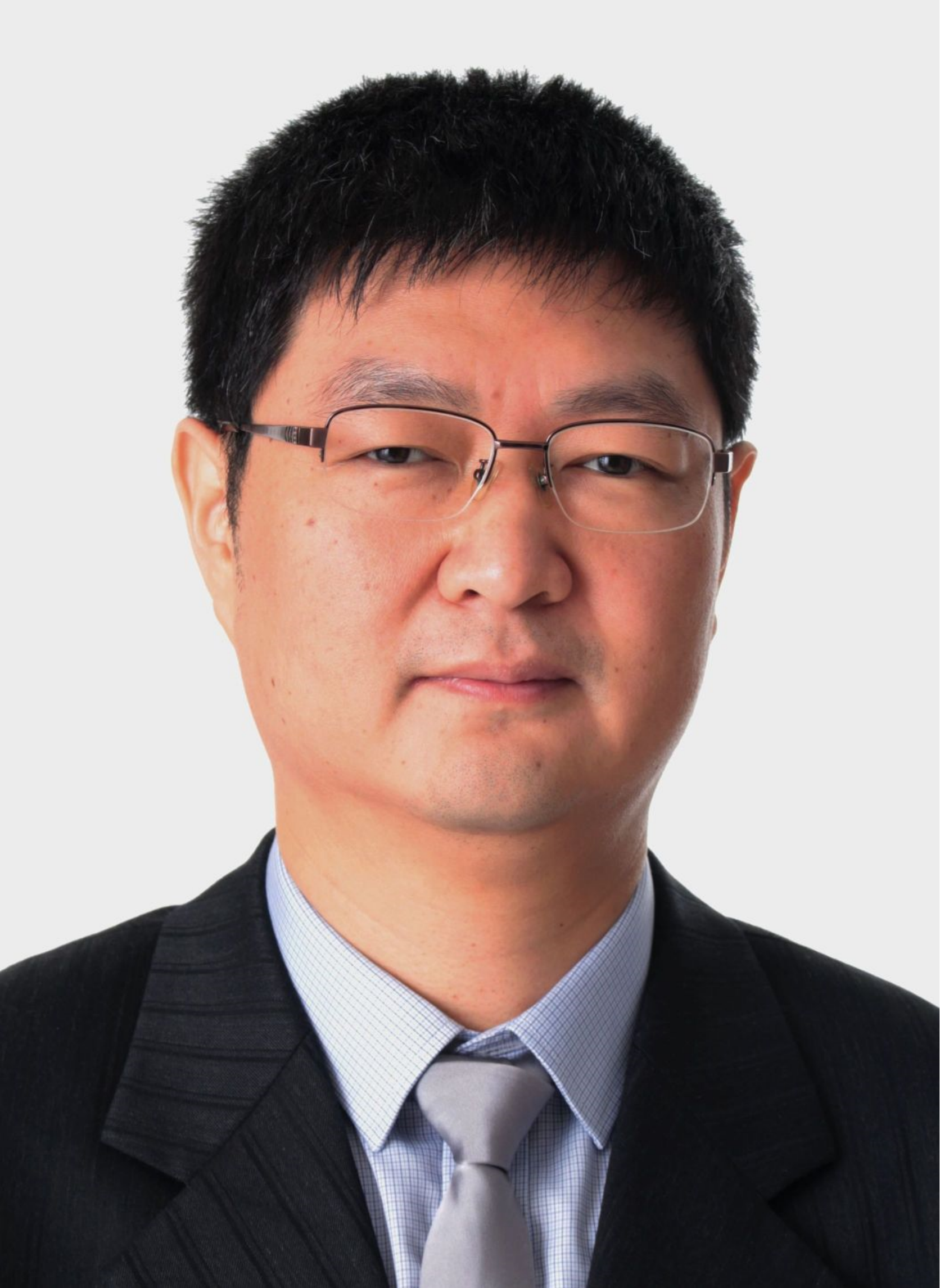}}]{Lei Zhang} (M'04, SM'14) received his B.Sc. degree in 1995 from Shenyang Institute of Aeronautical Engineering, Shenyang, P.R. China, and M.Sc. and Ph.D degrees in Control Theory and Engineering from Northwestern Polytechnical University, Xi’an, P.R. China, respectively in 1998 and 2001, respectively. From 2001 to 2002, he was a research associate in the Department of Computing, The Hong Kong Polytechnic University. From January 2003 to January 2006 he worked as a Postdoctoral Fellow in the Department of Electrical and Computer Engineering, McMaster University, Canada. In 2006, he joined the Department of Computing, The Hong Kong Polytechnic University, as an Assistant Professor. Since July 2015, he has been a Full Professor in the same department. His research interests include Computer Vision, Pattern Recognition, Image and Video Processing, and Biometrics, etc. Prof. Zhang has published more than 200 papers in those areas. As of 2016, his publications have been cited more than 20,000 times in the literature. Prof. Zhang is an Associate Editor of IEEE Trans. on Image Processing, SIAM Journal of Imaging Sciences and Image and Vision Computing, etc. He is a ``Highly Cited Researcher'' selected by Thomson Reuters. More information can be found in his homepage http://www4.comp.polyu.edu.hk/~cslzhang/.
\end{IEEEbiography}

\end{document}